\documentclass[twoside]{article}

\usepackage[accepted]{aistats2019}

\usepackage[utf8]{inputenc} 
\usepackage[T1]{fontenc}    
\usepackage[colorlinks]{hyperref}
\usepackage{url}            
\usepackage{amsfonts}       
\usepackage{nicefrac}       
\usepackage{microtype}
\usepackage{graphicx}

\usepackage{wrapfig,booktabs}
\usepackage{subcaption}

\usepackage{amsmath} 
\usepackage{kky}
\usepackage[ruled,vlined]{algorithm2e}
\usepackage{algorithmic}
\usepackage{xspace}
\usepackage{natbib}
\usepackage{xcolor}
\usepackage{varwidth}
\hypersetup{
	citecolor = {blue}
}

\newcommand*{\eg}{e.g.\@\xspace}
\newcommand*{\ie}{i.e.\@\xspace}

\begin{document}
\twocolumn[

\aistatstitle{Implicit Kernel Learning}

\aistatsauthor{ Chun-Liang Li \And Wei-Cheng Chang \And  Youssef Mroueh \And Yiming Yang \And Barnab{\'a}s P{\'o}czos }

\aistatsaddress{ \texttt{ \{chunlial, wchang2, yiming, bapoczos\}@cs.cmu.edu~~~~~~mroueh@us.ibm.com} \\ Carnegie Mellon University and IBM Research}]

\begin{abstract}
Kernels are powerful and versatile tools in machine learning and statistics. Although the notion of \emph{universal
kernels} and \emph{characteristic kernels} has been studied, kernel selection still greatly influences the empirical
performance.
While learning the kernel in a data driven way has been investigated,
in this paper we explore learning the spectral distribution of kernel via
implicit generative models parametrized by deep neural networks.
We called our method \emph{Implicit Kernel Learning} (IKL).
The proposed framework is simple to train and inference is performed via sampling random Fourier features.
We investigate two applications of the proposed IKL as examples, including generative adversarial networks with
MMD (MMD GAN) and standard supervised learning.
Empirically, MMD GAN with IKL outperforms vanilla predefined kernels on both image and text generation
benchmarks; using IKL with Random Kitchen Sinks also leads to substantial improvement over existing state-of-the-art
kernel learning algorithms on popular supervised learning benchmarks. 
Theory and conditions for using IKL in both applications are also studied
as well as connections to previous state-of-the-art methods.
\end{abstract}

\section{Introduction}

Kernel methods are among the essential foundations in machine learning and have been extensively studied in the past decades.
In supervised learning, kernel methods allow us to learn non-linear hypothesis. 
They also play a crucial role in statistics. 
Kernel maximum mean discrepancy (MMD)~\citep{Gretton2012ktest} is a powerful two-sample test, which is based 
on a statistics computed via kernel functions.
Even though there is a surge of deep learning in the past years, several successes have been shown by
kernel methods and deep feature extraction.
\citet{wilson2016deep} demonstrate state-of-the-art performance by incorporating deep
learning, kernel and Gaussian process. \citet{Li2015GMM, DziugaiteRG15} use MMD to train deep
generative models for complex datasets. 

In practice, however, kernel selection is always an important step. Instead of choosing by a heuristic, 
several works have studied \emph{kernel learning}. 
Multiple kernel learning (MKL)~\citep{bach2004multiple, lanckriet2004learning, bach2009exploring, gonen2011multiple, duvenaud2013structure} is one of the pioneering frameworks
to combine predefined kernels. 
One recent kernel learning development is to learn kernels via learning spectral distributions
(Fourier transform of the kernel). \citet{wilson2013gaussian} model spectral distributions via a mixture of Gaussians,
which can also be treated as an extension of linear combination of kernels~\citep{bach2004multiple}.  
\citet{oliva2016bayesian} extend it to Bayesian non-parametric models. 
In addition to model spectral distribution with \emph{explicit} density models aforementioned, many works
optimize the sampled random features or its weights (\eg \citet{buazuavan2012fourier, yang2015carte, sinha2016learning,
chang2017data, bullins2017not}).
The other orthogonal approach to modeling spectral distributions is learning feature maps for standard kernels (\eg
Gaussian).  
Feature maps learned by deep learning lead to state-of-the-art
performance on different tasks~\citep{hinton2008using, wilson2016deep, Li2017mmdgan}. 

In addition to learning effective features, implicit generative models via deep learning also lead to 
promising performance in learning distributions of complex data~\citep{Goodfellow14GAN}. 
Inspired by its recent success,
we propose to model kernel spectral distributions with implicit generative models in a
data-driven fashion, which we call \emph{Implicit Kernel Learning} (IKL). 
IKL provides a new route to modeling spectral distributions by learning sampling processes of the spectral densities, which is
under explored by previous works aforementioned.

In this paper, we start from studying the generic problem formulation of IKL, and propose an easily implemented, trained and evaluated
neural network parameterization which satisfies Bochner's theorem (Section~\ref{sec:kernel}). 
We then demonstrate two example applications of the proposed IKL. 
Firstly, 
we explore MMD GAN~\citep{Li2017mmdgan} with IKL on learning to generate images and text (Section~\ref{sec:gan}).
Secondly, we consider a standard two-staged supervised learning task with Random Kitchen Sinks~\citep{sinha2016learning}
(Section~\ref{sec:classification}).
The conditions required for training IKL and its theoretical guarantees in both tasks are also studied.
In both tasks, we show that IKL leads to competitive or better performance than heuristic kernel
selections and existing approaches
modeling kernel spectral densities. It demonstrates the potentials of learning more powerful kernels via deep
generative models. 
Finally, we discuss the connection with existing works in Section~\ref{sec:discussion}.

\section{Kernel Learning}
\label{sec:kernel}

Kernels have been used in several applications with success,
including supervised learning, unsupervised learning, and hypothesis testing. They have also been combined with deep
learning in different applications~\citep{mairal2014convolutional, Li2015GMM, DziugaiteRG15,wilson2016deep, mairal2016end}. 
Given data $x \in \RR^d$, kernel methods compute the inner product of the feature
transformation $\phi(x)$ in a high-dimensional Hilbert space $H$ 
via a kernel function $k : \Xcal \times \Xcal \rightarrow \RR$, which is defined as
$k(x,x') = \langle \phi(x), \phi(x') \rangle_{H}$,
where $\phi(x)$ is usually high or even infinitely dimensional. If $k$ is shift invariant (\ie
$k(x,y)=k(x-y)$), we can represent $k$ as an expectation with respect to a spectral distribution $\PP_k(\omega)$.
\paragraph{Bochner's theorem \citep{Rudin_book_11}}
	A continuous, real valued, symmetric and shift-invariant function
	$k$ on $\RR^d$ is a positive definite kernel if and only if 
	there is a positive finite measure $\PP_k(\omega)$ such that
	\[
		k(x-x') = \int_{\RR^d} e^{i\omega^\top (x-x')}d\PP_k(\omega) = \EE_{\omega\sim \PP_k}\left[ e^{i\omega^\top
		(x-x')} \right]. \label{eq:bochner}
	\]

\subsection{Implicit Kernel Learning}
We restrict ourselves to learning shift invariant kernels. According to that, learning kernels is equivalent to learning a
spectral distribution by optimizing
\begin{equation}
\begin{array}{l}
	  \displaystyle \arg\max_{k\in \Kcal} \sum_{i=1} \EE_{x\sim \PP_i,x'\sim \QQ_i}\left[ F_i(x, x')k(x,x') \right]  = \\
 \displaystyle \arg\max_{k\in \Kcal}   \sum_{i=1} \EE_{x\sim \PP_i,x'\sim \QQ_i}\left[ F_i( x, x')\EE_{\omega\sim \PP_k}\left[ e^{i\omega^\top (x-x')} \right] \right],
\end{array}
\label{eq:kernel_spec}
\end{equation}
where $F$ is a task-specific objective function and $\Kcal$ is a set of kernels. \eqref{eq:kernel_spec} covers many popular objectives, such as kernel
alignment~\citep{gonen2011multiple} and
MMD distance~\citep{Gretton2012ktest}. 
Existing works \citep{wilson2013gaussian, oliva2016bayesian} learn the spectral density $\PP_k(\omega)$ 
with \emph{explicit} forms via parametric or non-parametric models. 
When we learn kernels via~\eqref{eq:kernel_spec}, it may not be necessary to model the density of $\PP_k(\omega)$, 
as long as we are able to estimate kernel evaluations $k(x-x') = \EE_\omega[e^{i\omega^\top(x-x')}]$ via 
sampling from $\PP_k(\omega)$~\citep{Rahimi_NIPS_07}.
Alternatively, \emph{implicit probabilistic (generative) models} define a stochastic procedure that can generate (sample) data from
$\PP_k(\omega)$ without modeling $\PP_k(\omega)$. Recently, the neural implicit generative models~\citep{mackay1995bayesian}
regained attentions with promising results~\citep{Goodfellow14GAN} and 
simple sampling procedures.  We first sample $\nu$ from a base distribution $\PP(\nu)$ which is known (\eg Gaussian distribution), then use a
deterministic function $h_\psi$ parametrized by $\psi$, to transform $\nu$ into $\omega=h_\psi(\nu)$, 
where $\omega$ follows the complex target distribution $\PP_k(\omega)$. 
Inspired by the success of deep implicit generative models~\citep{Goodfellow14GAN}, 
we propose an \emph{Implicit Kernel Learning (IKL)} method by modeling
$\PP_k(\omega)$ via an implicit generative model $h_\psi(\nu)$, where $\nu\sim \PP(\nu)$, which results in
\begin{equation}
k_\psi(x, x') =\EE_{\nu}\left[e^{ih_\psi(\nu)^\top (x-x')}\right]
\label{eq:ikl}
\end{equation}

and reducing~\eqref{eq:kernel_spec} to solve
\begin{equation}
	\begin{array}{c}
		\displaystyle \arg\max_{\psi} \sum_{i=1} \EE_{x\sim \PP_i,x'\sim \QQ_i}\left[F_i(x,x') \EE_\nu\left( e^{ih_\psi(\nu)^\top(x-x')} \right)\right].\\
	\end{array}
\label{eq:IKL}
\end{equation}
The gradient of ~\eqref{eq:IKL} can be represented as 
\[
  \displaystyle \sum_{i=1} \EE_{x\sim \PP_i,x'\sim \QQ_i} \EE_{\nu}\left[ \nabla_\psi F_i(x,x') e^{ih_\psi(\nu)^\top (x-x')} \right].
\]
Thus, ~\eqref{eq:IKL} can be optimized via sampling $x, x'$ from data and $\nu$ from the base distribution to estimate gradient as shown above (SGD) in every iteration. 
Next, we discuss the parametrization of $h_\psi$ to satisfy Bochner's Theorem,
and describe how to evaluate IKL kernel in practice.

\paragraph{Symmetric $\PP_k(\omega)$}
To result in real valued kernels, the spectral density has to be symmetric, where $\PP_k(\omega)=\PP_k(-\omega)$. Thus,
we parametrize $h_\psi(\nu) = \mbox{\texttt{sign}}(\nu)\circ\tilde{h}_{\psi}(\mbox{\texttt {abs}}(\nu))$, where $\circ$ is
the Hadamard product and $\tilde{h}_\psi$ can be any unconstrained function if the base distribution $\PP(\nu)$ is symmetric (\ie
$\PP(\nu)=\PP(-\nu)$), such as standard normal distributions.

\paragraph{Kernel Evaluation} Although there is usually no closed form for the kernel evaluation $k_\psi(x, x')$ 
in~\eqref{eq:ikl}
with fairly complicated $h_\psi$, we can evaluate (approximate) $k_\psi(x,x')$ via sampling 
finite number of random Fourier features $\hat{k}_\psi(x, x')=\hat{\phi}_{h_\psi}(x)^\top\hat{\phi}_{h_\psi}(x')$, where
$\hat{\phi}_{h_\psi}(x)^\top = [\phi(x; h_\psi(\nu_1)), \dots, \phi(x; h_\psi(\nu_m))]$, and $\phi(x; \omega)$ is
the evaluation on $\omega$ of the Fourier transformation $\phi(x)$~\citep{Rahimi_NIPS_07}.

Next, we demonstrate two example applications covered by~\eqref{eq:IKL}, where we can apply IKL, including kernel alignment and maximum mean
discrepancy (MMD).

\section{MMD GAN with IKL}
\label{sec:gan}
Given $\{x_i\}_{i=1}^n \sim \PP_\Xcal$,
instead of estimating the density $\PP_\Xcal$, Generative
Adversarial Network (GAN)~\citep{Goodfellow14GAN} is an implicit generative model, which learns a generative
network~$g_\theta$ (generator). The generator $g_\theta$ transforms a base distribution $\PP_\Zcal$ over $\Zcal$ into $\PP_\theta$
to approximate $\PP_\Xcal$, where $\PP_\theta$ is the distribution of~$g_\theta(z)$ and~$z\sim \PP_\Zcal$.
During the training, GAN alternatively estimates a \emph{distance} $D(\PP_\Xcal\|\PP_\theta)$ between~$\PP_\Xcal$ and $\PP_\theta$, and
updates $g_\theta$ to minimize $D(\PP_\Xcal\|\PP_\theta)$. 
Different probability metrics have been studied~\citep{Goodfellow14GAN, Li2015GMM, DziugaiteRG15, NowozinCT16fgan, 
ArjovskyCB17wgan, Mroueh2017mcgan, Li2017mmdgan, mroueh2017fisher, gulrajani2017improved, mroueh2017sobolev,
arbel2018gradient} for
training GANs.

Kernel maximum mean discrepancy (MMD) is a probability metric, which is commonly used in two-sample-test to
distinguish two distributions with finite samples~\citep{Gretton2012ktest}. Given a kernel $k$,  
the MMD between $\PP$ and $\QQ$ is defined as
\begin{equation}
	M_k(\PP, \QQ) = \EE_{\PP,\PP}[ k(x,x') ] -2\EE_{\PP,\QQ}[k(x,y)] + \EE_{\QQ, \QQ}[ k(y,y') ].
	\label{eq:mmd}
\end{equation}
For characteristic kernels, $M_k(\PP, \QQ)=0$ \emph{iff} $\PP=\QQ$.
\citet{Li2015GMM, DziugaiteRG15} train the generator $g_\theta$ by optimizing $\min_\theta {M}_k(\PP_\Xcal,
\PP_\theta)$ with a Gaussian kernel~$k$. \citet{Li2017mmdgan} propose MMD GAN,
which trains $g_\theta$ via $\min_\theta \max_{k\in \Kcal}{M}_k(\PP_\Xcal,\PP_\theta)$, where $\Kcal$ is a pre-defined
set of kernels. The intuition is to learn a kernel $\argmax_{k\in\Kcal} {M}_k(\PP_\Xcal,\PP_\theta)$, which has a
stronger signal (\ie larger distance when $\PP_\Xcal\neq \PP_\theta$) to train $g_\theta$. 
Specifically, \citet{Li2017mmdgan} consider a 
composition kernel $k_{\varphi}$ 
which combines Gaussian kernel $k$ and a neural network~$f_\varphi$ as $k_\varphi = k\circ f_\varphi$, where
\begin{equation}
	 k_{\varphi}(x, x') = \exp(-\|f_\varphi(x)-f_\varphi(x)'\|^2).
	\label{eq:composition_k}
\end{equation}
The MMD GAN objective then becomes 
	$\min_\theta \max_\varphi M_{\varphi}(\PP_\Xcal, \PP_\theta)$.

\subsection{Training MMD GAN with IKL}
Although the composition kernel with a learned
feature embedding $f_\varphi$ is powerful, choosing a good base kernel $k$ is still crucial in
practice~\citep{binkowski2018demystifying}. 
Different base kernels for MMD GAN, such as rational quadratic kernel~\citep{binkowski2018demystifying} and
distance kernel~\citep{bellemare2017cramer}, have been studied.
Instead of choosing it by hands, we propose to learn the base kernel by IKL, which 
extend~\eqref{eq:composition_k} to be  $k_{\psi,\varphi}=k_\psi\circ f_\varphi$ with the form
\begin{equation}
	k_{\psi,\varphi}(x, x') = \EE_\nu\left[ e^{ih_\psi(\nu)^\top(f_\varphi(x)-f_\varphi(x'))  } \right].
	\label{eq:ikl_mmdgan}
\end{equation}
We then extend the MMD GAN objective to be 
\begin{equation}
	\min_\theta \max_{\psi,\varphi} M_{\psi,\varphi}(\PP_\Xcal, \PP_\theta),
	\label{eq:mmdgan_ikl}
\end{equation}
where $M_{\psi,\varphi}$ is the MMD distance~\eqref{eq:mmd} with the IKL kernel~\eqref{eq:ikl_mmdgan}.
Clearly, for a given $\varphi$, the maximization over $\psi$ in~\eqref{eq:mmdgan_ikl} can
be represented as~\eqref{eq:kernel_spec} by letting $F_1(x,x')=1$, $F_2(x,y)=-2$ and
$F_3(y,y')=1$.
In what follows, we will use  for convenience $k_{\psi,\varphi}$, $k_{\psi}$ and $k_{\varphi}$ to denote kernels defined in~\eqref{eq:ikl_mmdgan},
\eqref{eq:ikl} and ~\eqref{eq:composition_k} respectively.

\subsection{Property of MMD GAN with IKL}

As proven by~\citet{arjovsky2017towards}, some probability distances adopted by existing works (\eg \citet{Goodfellow14GAN}) 
are not \emph{weak} (\ie $\PP_n \xrightarrow[]{D} \PP$ then $D(\PP_n\|\PP) \rightarrow 0$), which cannot provide better
signal to train $g_\theta$. Also, they usually 
suffer from discontinuity, hence it cannot be trained via gradient descent at certain points. 
We prove that $\max_{\psi,\varphi} M_{\psi,\varphi}(\PP_\Xcal, \PP_\theta)$  is a continuous and
differentiable objective in $\theta$ and \emph{weak} under mild assumptions as used in~\citep{ArjovskyCB17wgan, Li2017mmdgan}.
\begin{assumption} 
	\label{ass:cont} 
	$g_\theta(z)$ is  locally Lipschitz and differentiable in $\theta$; $f_\varphi(x)$ is Lipschitz in $x$ and
	$\varphi \in \Phi$ is compact. 
	$f_\varphi\circ g_\theta(z)$ is differentiable in $\theta$ and 
    there are local Lipschitz constants, which is independent of $\varphi$, such that $ \EE_{z \sim \PP_z}[L(\theta, z)]
	<+\infty$. The above assumptions are adopted by~\citet{ArjovskyCB17wgan}.
	Lastly, assume given any $\psi\in\Psi$, where $\Psi$ is compact, $k_\psi(x,x')=\EE_\nu\left[ e^{ih_\psi(\nu)^\top(x-x')}
	 \right]$ and $|k_\psi(x,x')|<\infty$ is differentiable and Lipschitz in $(x,x')$ which has an 
	 upper bound $L_k$ for Lipschitz constant of $(x,x')$ given different $\psi$.
\end{assumption}

\begin{theorem} \label{thm:cont}
Assume function $g_\theta$ and kernel $k_{\psi,\varphi}$ satisfy Assumption~\ref{ass:cont}, 
$\max_{\psi,\varphi} M_{\psi,\varphi}$ is weak, that is, $\max_{\psi,\varphi} M_{\psi,\varphi}(\PP_\Xcal,
\PP_n) \rightarrow 0 \Longleftrightarrow  \PP_n \xrightarrow[]{D} \PP_\Xcal$.
Also, $\max_{\psi,\varphi} M_{\psi,\varphi}(\PP_\Xcal, \PP_\theta)$ is continuous everywhere 
and differentiable almost everywhere in $\theta$.
\end{theorem}

\begin{lemma}
Assume $\Xcal$ is bounded. Let $x,x' \in \Xcal$,
$k_\psi(x,x')=\EE_\nu\left[ e^{ih_\psi(\nu)^\top(x-x')} \right]$ is Lipschitz in $(x,x')$ if
$\EE_\nu\left[\|h_\psi(\nu)\|^2\right] < \infty$, which is variance since $\EE_\nu\left[h_\psi(\nu)\right]=0.$
\label{lem:var}
\end{lemma}

We penalize~$\lambda_h(\EE_\nu\left[\|h_\psi(\nu)\|^2\right] -u )^2$ as an approximation of Lemma~\ref{lem:var} in practice to ensure that assumptions in
Theorem~\ref{thm:cont} are satisfied. 
The algorithm with IKL and gradient penalty~\citep{binkowski2018demystifying} is shown in Algorithm \ref{alg:mmdgan_IKL}.

\vspace{-1em}
\begin{algorithm} 
	\caption{MMD GAN with IKL} 
    \label{alg:mmdgan_IKL}
	\begin{algorithmic}
	    \STATE {\bfseries Input: }
            $\eta$ the learning rate, $B$ the batch size,
            $n_c$ number of $f,h$ updates per $g$ update,
            $m$ the number of basis,
            $\lambda_{GP}$ the coefficient of gradient penalty,
            $\lambda_h$ the coefficient of variance constraint.
        \STATE {{\bf Initial} parameter $\theta$ for $g$, $\varphi$ for $f$, $\psi$ for $h$}
        \STATE {{\bf Define} $\Lcal(\psi,\varphi) = M_{\psi,\varphi}(\PP_\Xcal, \PP_\theta)
            - \lambda_{GP} (\|\nabla_{\hat{x}} f_{\varphi}(\hat{x})\|_2 - 1 )^2
            - \lambda_h(\EE_{\nu}[\|h_\psi(\nu)\|^2] - u)^2$}
	    \WHILE {$\theta$ has not converged}
            \FOR{$t=1,\ldots,n_c$}
                \STATE{Sample $\{x_i\}_{i=1}^B \sim \PP(\Xcal), \
                \{z_j\}_{j=1}^B \sim \PP(\Zcal), \
                \{\nu_k\}_{k=1}^m \sim \PP(\nu)$}
                \STATE{ $(\psi,\varphi) \leftarrow \varphi + \eta\text{Adam}\left( (\psi,\varphi), \nabla_{\psi,\varphi}\Lcal(\psi,\varphi)\right) $ }
            \ENDFOR
            \STATE{Sample $\{x_i\}_{i=1}^B \sim \PP(\Xcal), \
            \{z_j\}_{j=1}^B \sim \PP(\Zcal), \
            \{\nu_k\}_{k=1}^m \sim \PP(\nu)$}
            \STATE{ $\theta \leftarrow \theta - \eta\text{Adam}(\theta, \nabla_{\theta} M_{\psi,\varphi}(\PP_\Xcal, \PP_\theta) ) $ }
	    \ENDWHILE
	\end{algorithmic}
\end{algorithm}
\vspace{-1em}

\subsection{Empirical Study}
\label{sec:gan_empirical}
We consider image and text generation tasks for quantitative evaluation.
For image generation, we evaluate the inception score~\citep{salimans2016improved} and FID score~\citep{heusel2017gans} on
CIFAR-10~\citep{krizhevsky2009learning}. 
We use DCGAN~\citep{radford2015unsupervised} and expands the output of $f_\varphi$ to be 16-dimensional
as~\citet{binkowski2018demystifying}.
For text generation, we consider a length-32 character-level generation task on Google Billion Words dataset. 
The evaluation is based on Jensen-Shannon divergence on empirical
4-gram probabilities (JS-4) of the generated sequence and the validation data as used by~\citet{gulrajani2017improved, 
heusel2017gans, mroueh2017sobolev}.
The model architecture follows~\citet{gulrajani2017improved} in using ResNet with 1D convolutions.
We train every algorithm $10,000$ iterations for comparison. 

For MMD GAN with fixed base kernels, we consider the mixture of Gaussian kernels 
$k(x,x')=\sum_{q}\exp(-\frac{\|x-x'\|^2}{2\sigma_q^2})$~\citep{Li2017mmdgan} and the mixture of RQ kernels
$k(x,x')=\sum_{q}(1+\frac{\|x-x'\|^2}{2\alpha_q})^{-\alpha_q}$. 
We tuned hyperparameters $\sigma_q$ and $\alpha_q$ for each kernel as reported in Appendix~\ref{sec:gan_hyp}.

Lastly, for learning base kernels, we compare IKL with SM kernel~\citep{wilson2013gaussian} 
$f_\varphi$, which learns mixture of Gaussians to model kernel spectral density.
It can also be treated as the \emph{explicit generative model counter part} of the proposed IKL.

In both tasks, $\PP(\nu)$, the base distribution of IKL, is a standard normal distribution and $h_\psi$ is a 3-layer MLP with $32$ hidden units for
each layer. Similar to the aforementioned mixture kernels, we consider the mixture of IKL kernel with the variance
constraints $\EE[\|h_\psi(\nu)\|^2] = 1/\sigma_q$, where $\sigma_q$ is the bandwidths for the mixture of Gaussian kernels. 
Note that if $h_\psi$ is an identity map, we recover the mixture of Gaussian kernels.
We fix $\lambda_h$ to be $10$ and resample $m=1024$ random features for IKL in every iteration.
For other settings, we follow~\cite{binkowski2018demystifying} and the hyperparameters can be found in Appendix~\ref{sec:gan_hyp}.

\subsubsection{Results and Discussion}

We compare MMD GAN with the proposed IKL and different fixed kernels.
We repeat the experiments $10$ times and report the average result
with standard error in Table~\ref{tb:gan_results}. 
Note that for inception score the larger the better; while JS-4 the smaller the better. 
We also report WGAN-GP results as a reference. 
Since FID score results~\citep{heusel2017gans} is consistent with inception score and does not change our discussion,
we put it in Appendix~\ref{sec:addition} due to space limit. 
Sampled images on larger datasets are shown in Figure~\ref{fig:gan-samples}. 
\begin{table}[h]
    \begin{tabular}{c|c|c}
        \toprule
        Method & Inception Scores $(\uparrow)$ & JS-4 $(\downarrow)$ \\
        \midrule
        Gaussian & $6.726 \pm 0.021$  & $0.381 \pm 0.003$\\
        RQ & $6.785 \pm 0.031$ &  $0.463 \pm 0.005$\\
		SM & $6.746\pm 0.031$ & $0.378 \pm 0.003$ \\
        IKL & $\mathbf{6.876 \pm 0.018}$ & $\mathbf{0.372 \pm 0.002}$\\
        \midrule
        WGAN-GP & $6.539 \pm 0.034$  & $0.379 \pm 0.002$ \\
        \bottomrule
    \end{tabular}
    \caption{Inception scores and JS-4 divergece results.}
	\vspace{-4mm}
    \label{tb:gan_results}
\end{table}

\begin{figure*}[t]
    \centering
    \begin{subfigure}[t]{0.31\linewidth}
    	\includegraphics[width=\linewidth]{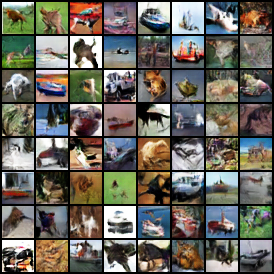}
    \end{subfigure}
    \hfill%
    \begin{subfigure}[t]{0.31\linewidth}
    	\includegraphics[width=\linewidth]{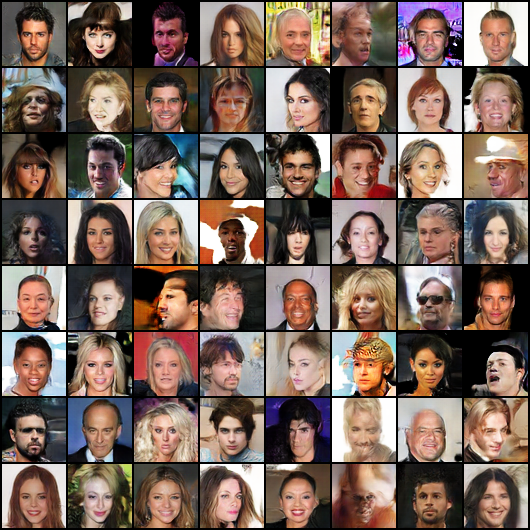}
    \end{subfigure}
    \hfill%
    \begin{subfigure}[t]{0.31\linewidth}
    	\includegraphics[width=\linewidth]{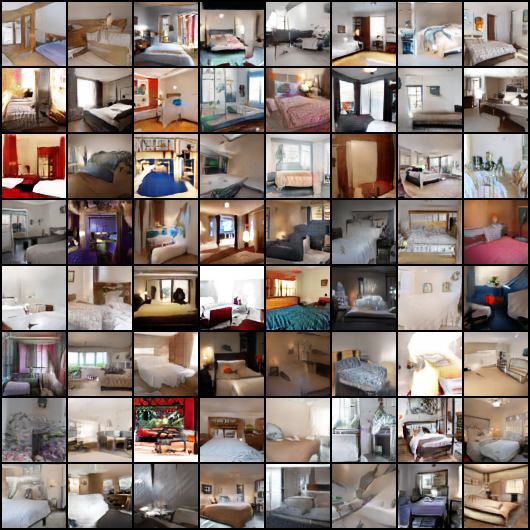}
    \end{subfigure}
	\vspace{-2mm}
	\caption{Samples generated by MMDGAN-IKL on CIFAR-10, CELEBA and LSUN dataset.}
	\label{fig:gan-samples}
	\vspace{-3mm}
\end{figure*}

\paragraph{Pre-defined Kernels}
\citet{binkowski2018demystifying} show RQ kernels outperform Gaussian and energy distance kernels on image generation. 
Our empirical results agree with such finding: RQ kernels achieve $6.785$ inception score while for Gaussian kernel it is
$6.726$, as shown in the left column of Table \ref{tb:gan_results}. In text generation, nonetheless, RQ kernels only
achieve $0.463$ JS-4 score\footnote{For RQ kernels, we searched $10$ possible hyperparameter settings and reported the
best one in Appendix, to ensure the unsatisfactory performance is not caused by the improper parameters.}
and are not on par with $0.381$ acquired by Gaussian kernels, even though it is still slightly worse than WGAN-GP.
These results imply \emph{kernel selection is task-specific}. On the other hand, the proposed IKL learns  
kernels in a data-driven way, which results in the best
performance in both tasks. In CIFAR-10, although Gaussian kernel is worse than RQ, IKL is still able to transforms
$\PP(\nu)$, which is Gaussian, into a powerful kernel, and outperforms RQ on inception scores ($6.876$ v.s. $6.785$).
For text generation, from Table \ref{tb:gan_results} and
Figure~\ref{fig:text_mmd}, we observe that IKL can further boost Gaussian into better kernels with substantial
improvement.
Also, we note that the difference between IKL and pre-defined kernels in Table~\ref{tb:gan_results} is significant
based on the $t$-test at 95\% confidence level.

\begin{figure}[h]
   \vspace{-4mm}
   \includegraphics[width=0.5\textwidth]{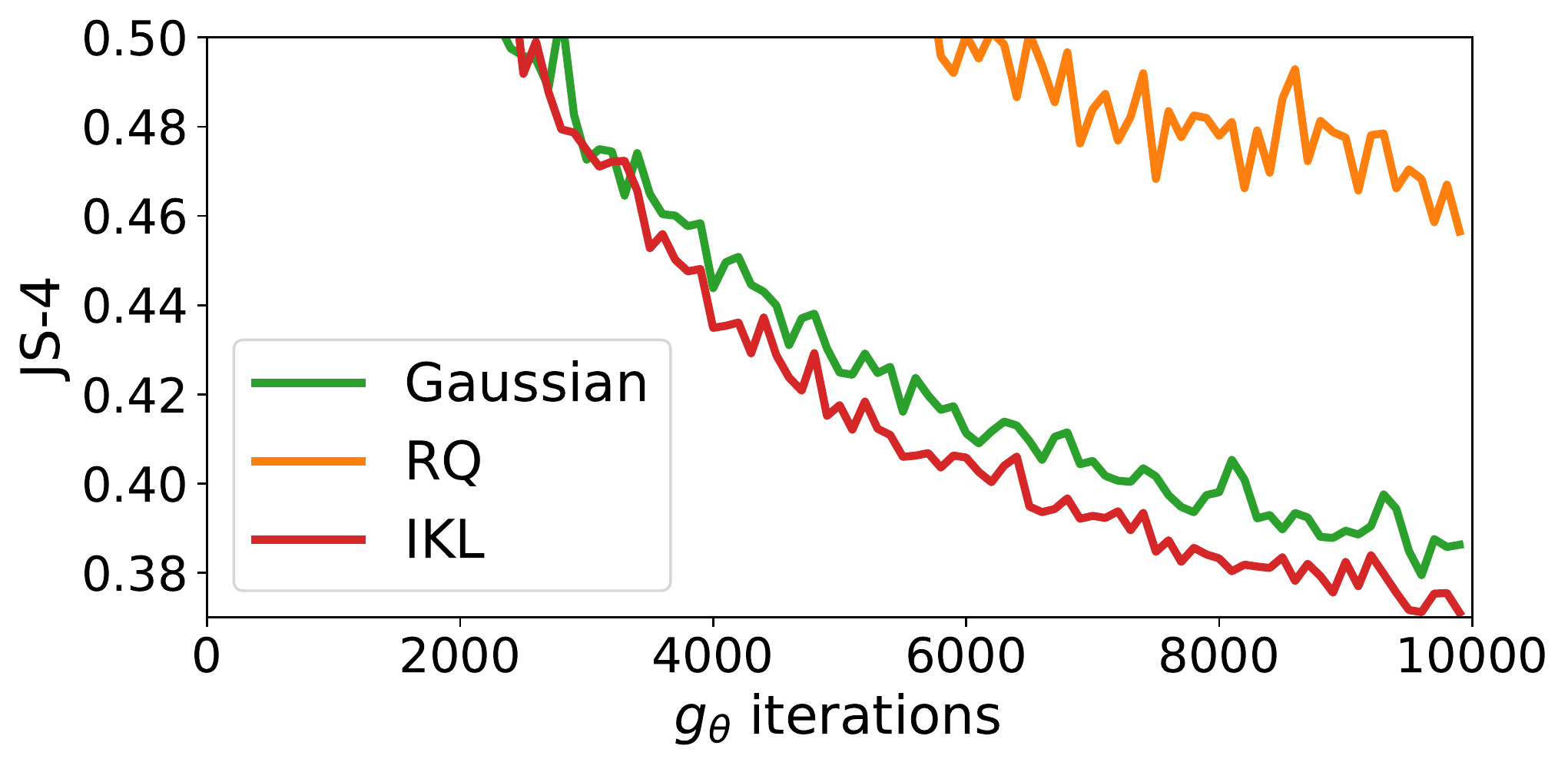}
   \vspace{-6mm}
   \caption{Convergence of MMD GANs with different kernels on text generation.}
   \label{fig:text_mmd}
   \vspace{-4mm}
\end{figure}

\paragraph{Learned Kernels}
The SM kernel~\citep{wilson2013gaussian}, which learns the spectral density via mixture of Gaussians, does not
significantly outperforms Gaussian kernel as shown in Table~\ref{tb:gan_results}, since ~\citet{Li2017mmdgan}
already uses equal-weighted mixture of Gaussian formulation. It suggests that proposed IKL can learn more complicated
and effective spectral distributions than simple mixture models. 

\paragraph{Study of Variance Constraints}
In Lemma~\ref{lem:var}, we prove bounding variance $\EE[\|h_\psi(\nu)\|^2]$ guarantees $k_\psi$ to be Lipschitz as required in
Theorem~\ref{thm:cont}. We investigate the importance of this constraint.
In Figure~\ref{fig:novar}, we show the training objective (MMD), $\EE[\|h_\psi(\nu)\|^2]$ and the JS-4 divergence for
training MMD GAN (IKL) without variance constraint, \ie $\lambda_h=0$.
We could observe the variance keeps going up without constraints, which leads exploded MMD values. 
Also, when the exploration is getting severe, the JS-4 divergence starts increasing, which implies MMD cannot provide
meaningful signal to $g_\theta$. The study justifies the validity of Theorem~\ref{thm:cont} and Lemma~\ref{lem:var}.

\begin{figure}[h]
   \includegraphics[width=0.5\textwidth]{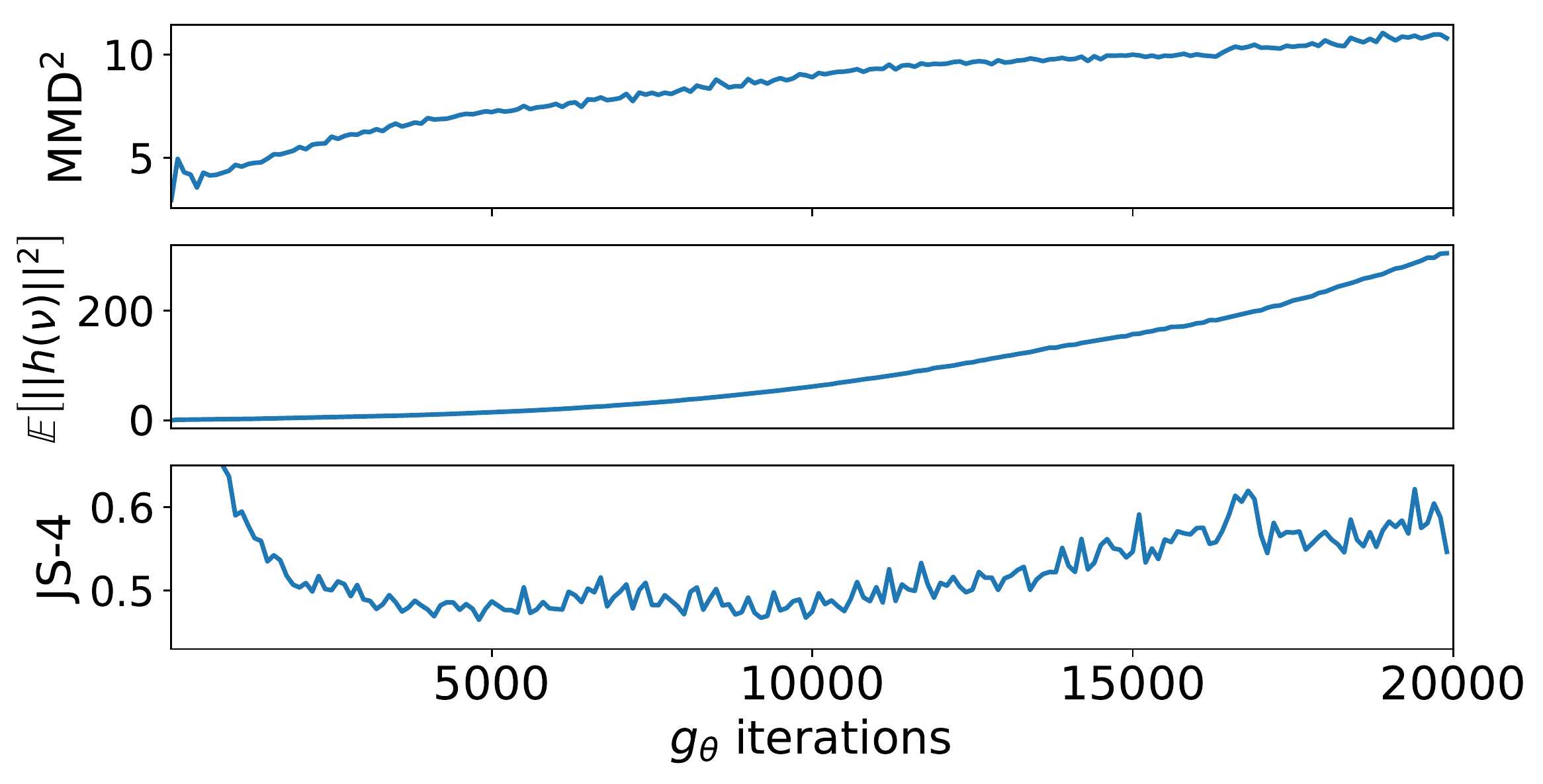}
   \caption{ Learning MMD GAN (IKL) without the variance constraint on Google Billion Words datasets for text generation.  }
   \label{fig:novar}
\end{figure}

\paragraph{Other Studies}
One concern of the proposed IKL is the computational overhead introduced by sampling random features as well as using
more parameters to model $h_\psi$. 
Since we only use small network to model $h_\psi$, the increased computation overhead is almost negligible under GPU
parallel computation. The detailed comparison can be found in Appendix~\ref{sec:comp}.
We also compare IKL with \citet{bullins2017not}, which can
be seen as a variant of IKL without $h_\psi$, and studt the variance constraint. Those additional discussions can be found in
Appendix~\ref{sec:addition}.

\section{Random Kitchen Sinks with IKL}
\label{sec:classification}
\citet{rahimi2009weighted} propose \emph{Random Kitchen Sinks} (RKS) as follows.
We sample $\omega_i\sim \PP_k(\omega)$ and transform $x\in\RR^d$
into $\hat{\phi}(x)=[\phi(x; \omega_1),\dots,\phi(x; \omega_M)]$, where $\sup_{x,\omega} |\phi(x; \omega)| < 1$. 
We then
learn a classifier on the transformed features $\hat{\phi}(x;\omega)$. 
Kernel methods with random features~\citep{Rahimi_NIPS_07} is an example of RKS,
where $\PP_k(\omega)$ is the spectral distribution of the kernel and
$\phi(x;\omega) = \big[ \cos(\omega^\top x), \sin(\omega^\top x) \big]$.
We  usually learn a model $\wb$ by solving 
\begin{equation}
	\arg\min_\wb \frac{\lambda}{2}\|\wb\|^2 + \frac{1}{n}\sum_{i=1}^n \ell\left( \wb^\top \hat{\phi}(x_i) \right).
	\label{eq:erm}
\end{equation}
If $\ell$ is a convex loss function, the objective \eqref{eq:erm} can be solved efficiently to global optimum.

Spectral distributions $\PP_k$ are usually set as a parameterized form, such as Gaussian distributions, but the 
selection of $\PP_k$ is important in practice. 
If we consider RKS as kernel methods with random features, then selecting $\PP$ is equivalent to
the well-known kernel selection (learning) problem for supervised learning~\citep{gonen2011multiple}. 

\paragraph{Two-Stage Approach}
We follows~\citet{sinha2016learning} to consider kernel learning for RKS with a two-stage approach. 
In stage 1, we consider kernel
alignment~\citep{cristianini2002kernel} of the form,
$\argmax_{k\in\Kcal} \EE_{(x,y), (x',y')}\sum_{i\neq j} yy' k(x, x')$.
By parameterizing $k$ via the implicit generative model $h_\psi$ as in Section~\ref{sec:kernel},
we have the following problem:
\begin{equation}
	\argmax_{\psi} \EE_{(x,y), (x',y')} yy'\EE_\nu\left[ 
	e^{ih_\psi(\nu)^\top(x-x') },
	\right].
	\label{eq:k_align}
\end{equation}
which can be treated as~\eqref{eq:kernel_spec} with $F_1(x, x')=yy'$.
After solving~\eqref{eq:k_align}, we learn a \emph{sampler} $h_\psi$ where we can easily sample.
Thus, in stage 2, we thus have the advantage of solving a convex problem~\eqref{eq:erm} in RKS with IKL.
The algorithm is shown in Algorithm~\ref{alg:super_IKL}.

\begin{algorithm}
	\caption{Random Kitchen Sinks with IKL} 
    \label{alg:super_IKL}
	\begin{algorithmic}
	\STATE {\bf Stage 1: Kernel Learning }
	\STATE {\bfseries Input: }$X=\{(x_i, y_i)\}_{i=1}^n$, the batch size $B$ for data and $m$
	for random feature, learning rate $\eta$
	\STATE {Initial parameter $\psi$ for $h$}
	\WHILE {$\psi$ has not converged or reach maximum iters}
		\STATE{Sample $\{(x_i, y_i)\}_{i=1}^B \subseteq X. \ $
        Fresh sample $\{\nu_j\}_{j=1}^m \sim \PP(\nu)$}
		\STATE{ $g_\psi \leftarrow \nabla _\psi \frac{1}{B(B-1)}\sum_{i\neq i'} y_iy_{i'} \frac{1}{m}\sum_{j=1}^m
		e^{ih_\psi(\nu_j)^\top(x_i-x_{i'}))} $ }
		\STATE{ $\psi \leftarrow \psi - \eta\text{Adam}(\psi, g_\psi) $ }
	\ENDWHILE

	\STATE {\bf Stage 2: Random Kitchen Sinks}
	\STATE{ Sample $\{\nu_i\}_{i=1}^M\sim \PP(\nu)$, note that $M$ is not necessarily equal to $m$}
	\STATE {Transform $X$ into $\phi(X)$ via $h_\psi$ and $\{\nu_i\}_{i=1}^M$}
	\STATE Learn a linear classifier on $(\phi(X), Y)$
	\end{algorithmic}
\end{algorithm}
Note that in stage 1, we resample $\{\nu_j\}_{j=1}^m$ in every iteration to train an implicit generative model $h_\psi$. 
The advantage of Algorithm~\ref{alg:super_IKL} is the random features used in kernel learning and RKS 
can be \emph{different}, which allows us to use \emph{less} random features in kernel learning (stage 1), and
sample \emph{more} features for RKS (stage 2). 

One can also \emph{jointly} train both feature mapping $\omega$ and the
model parameters $\wb$, such as neural networks.
We remark that our intent is not to show state-of-the-art results on supervised learning, 
on which deep neural networks dominate~\citep{krizhevsky2012imagenet, he2016deep}.
We use RKS as a protocol to study kernel learning and the proposed IKL, which still has competitive performance with
neural networks on some tasks~\citep{rahimi2009weighted, sinha2016learning}. 
Also, the simple procedure of RKL with IKL allows us to provide some theoretical guarantees of the performance, which is
sill challenging of deep learning models. 

\paragraph{Comparison with Existing Works}
\citet{sinha2016learning} learn non-uniform weights
for $M$ random features via kernel alignment in stage 1 then using these optimized features in RKS in the stage 2. 
Note that the random features used in stage 1 has to be the same as the ones in stage 2.
A jointly training of feature mapping and classifier can be treated as a
2-layer neural networks~\citep{buazuavan2012fourier, alber2017empirical, bullins2017not}. 
Learning kernels with aforementioned works will be more costly if we want to use a
large number of random features for training classifiers. 
In contrast to implicit generative models, ~\citet{oliva2016bayesian} learn an explicit Bayesian nonparametric generative
model for spectral distributions, which requires specifically designed inference algorithms.
Learning kernels for~\eqref{eq:erm} in dual form without random features has also been proposed. 
It usually require  costly steps, such as eigendecomposition of the Gram matrix~\citep{gonen2011multiple}.

\subsection{Empirical Study}
We evaluate the proposed IKL on both synthetic and benchmark binary classification tasks.
For IKL, $\PP(\nu)$ is standard Normal and $h_\psi$ is a $3$-layer MLP for all experiments.
The number of random features $m$ to train $h_\psi$ in Algorithm~\ref{alg:super_IKL} is fixed to be $64$. 
Other experiment details are described in Appendix~\ref{sec:classification_hyp}.

\begin{figure}
    \vspace{-1.0em}
    \begin{center}
    \begin{subfigure}[t]{0.475\linewidth}
      \includegraphics[width=\linewidth]{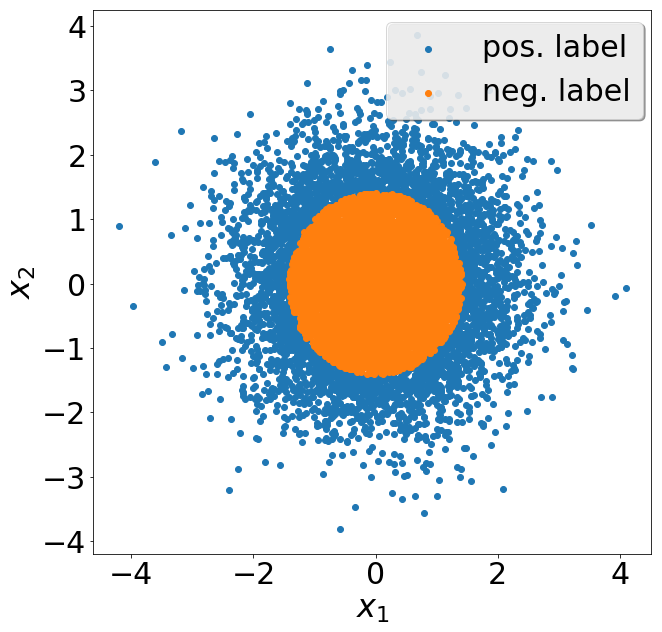}
    \end{subfigure}
    \hfill%
    \begin{subfigure}[t]{0.475\linewidth}
      \includegraphics[width=\linewidth]{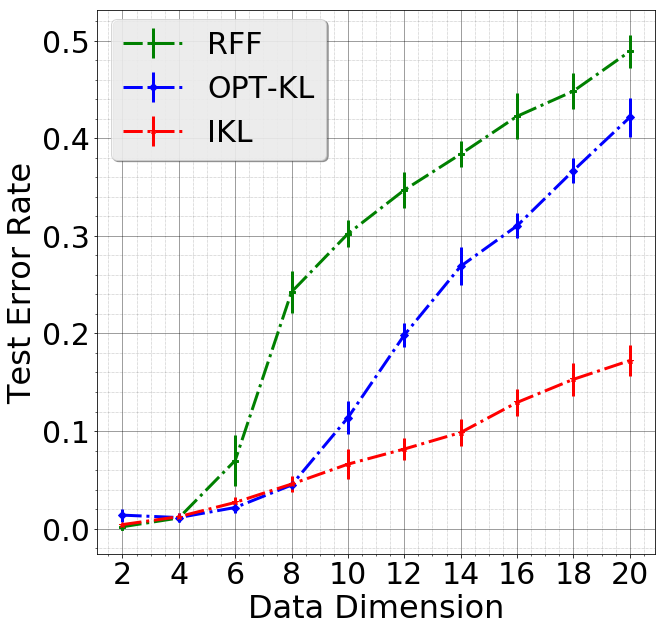}
    \end{subfigure}
    \end{center}
    \caption{Left figure is training examples when $d=2$.
    Right figure is the classification error v.s. data dimension.}
    \label{fig:exp-spv-fake}
\end{figure}

\paragraph{Kernel learning with a poor choice of $\PP_k(\omega)$}
We generate $\{x_i\}_{i=1}^n \sim \Ncal(0, I_d)$ with $y_i = \text{sign}(\|x\|_2 - \sqrt{d})$, where $d$ is the data
dimension. A two dimensional example is shown in Figure~\ref{fig:exp-spv-fake}.
Competitive baselines include random features (RFF) \citep{Rahimi_NIPS_07} as well as OPT-KL \citep{sinha2016learning}.
In the experiments, we fix $M=256$ in RKS for all algorithms. 
Since Gaussian kernels with the bandwidth $\sigma=1$ is known to be ill-suited for this task~\citep{sinha2016learning},
we  directly use random features from it for RFF and OPT-KL. Similarly, we set $\PP(\nu)$ to be standard normal
distribution as well.

The test error for different data dimension $d=\{2,4,\ldots,18,20\}$ is shown in Figure \ref{fig:exp-spv-fake}.
Note that RFF is competitive with OPT-KL and IKL when $d$ is small ($d \leq 6$),
while its performance degrades rapidly as $d$ increases, which is consistent with the observation
in~\citet{sinha2016learning}.
More discussion of the reason of failure can be referred to~\citet{sinha2016learning}.
On the other hand, although using standard normal as the spectral distribution is ill-suited for this task,  
both OPT-KL and IKL can adapt with data and learn to transform it into effective kernels and result in slower degradation with $d$. 

Note that OPT-KL learns the \emph{sparse} weights on the sampled random features ($M=256$).
However, the sampled random features can fail to contain informative ones, especially in high dimension~\citep{bullins2017not}.
Thus, when using limited amount of random features, OPT-IKL may result in worse performance than IKL in the high
dimensional regime in Figure~\ref{fig:exp-spv-fake}.

\paragraph{Performance on benchmark datasets}
Next, we evaluate our IKL framework on standard benchmark binary classification tasks.
Challenging label pairs are chosen from MNIST \citep{lecun1998gradient}
and CIFAR-10 \citep{krizhevsky2009learning} datasets;
each task consists of $10000$ training and $2000$ test examples.
For all datasets, raw pixels are used as the feature representation.
We set the bandwidth of RBF kernel by the median heuristic.
We also compare with~\citet{wilson2013gaussian}, the spectral mixture (SM) kernel, which uses Gaussian mixture
to learn spectral density and can be seen as the explicit generative model counterpart of IKL.
Also, SM kernel is a MKL variant with linear combination~\citep{gonen2011multiple}.
In addition, we consider 
the joint training of random features and model parameters, which can be treated as two-layer neural network (NN) and
serve as the lower bound of error for comparing different kernel learning algorithms.  

The test error versus different $M=\{2^6, 2^7, \ldots, 2^{13}\}$ in the second stage are shown in Figure \ref{fig:exp-spv-real}.
First, in light of computation efficiency, SM and the proposed IKL only sample
$m=64$ random features in each iteration in the first stage, and draws different
number of basis $M$ from the learned $h_{\psi}(\nu)$ for the second stage.
OPT-KL, on the contrary, the random features used in training and testing should be the same. 
Therefore, OPT-IKL needs to deal with $M$ random features in the training. It brings computation concern when $M$ is
large. 
In addition, IKL demonstrates improvement over 
the representative kernel learning method OPT-KL, 
especially significant on the challenging datasets such as CIFAR-10.
In some cases, IKL almost reaches the performance of NN, such as MNIST,
while OPT-KL degrades to RFF except for small number of basis $(M=2^6)$.
This illustrates the effectiveness of learning kernel spectral distribution
via the implicit generative model $h_{\psi}$. 
Also, IKL outperforms SM, which is consistent with the finding in Section~\ref{sec:gan} that IKL
can learn more complicated spectral distributions than simple mixture models (SM).

\begin{figure*}[t]
	\centering
    \begin{subfigure}[t]{0.245\linewidth}
    	\includegraphics[width=\linewidth]{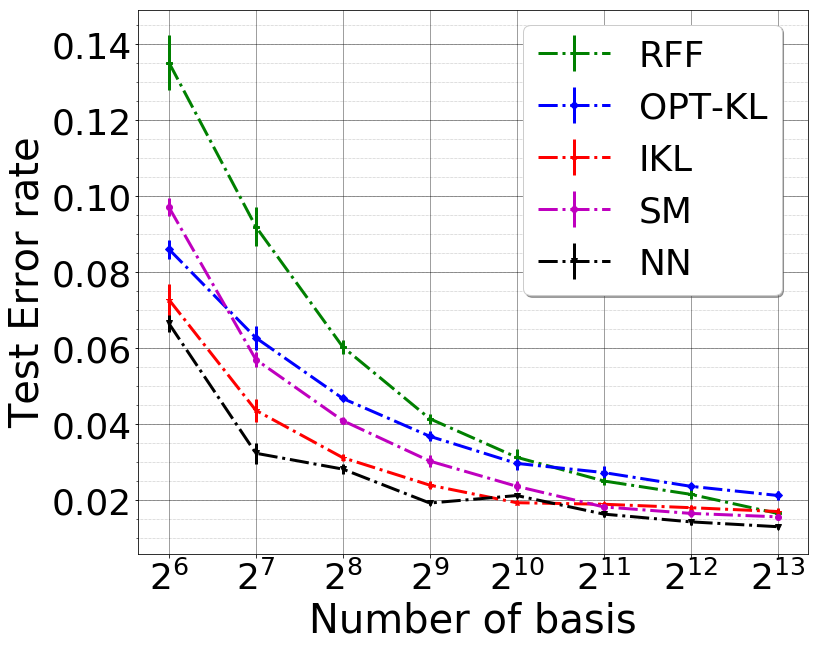}
        \caption{MNIST (4-9)}
    \end{subfigure}
    \hfill%
    \begin{subfigure}[t]{0.245\linewidth}
    	\includegraphics[width=\linewidth]{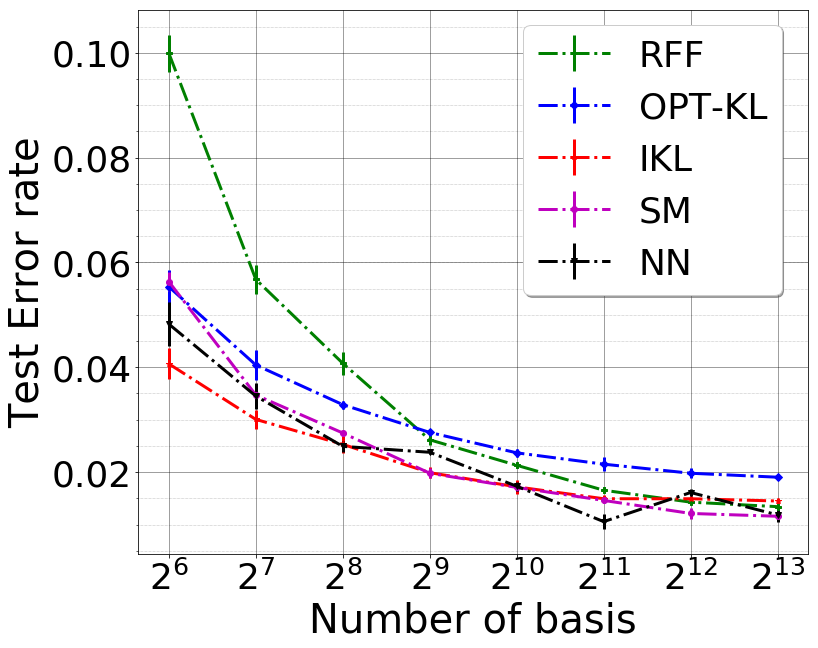}
        \caption{MNIST (5-6)}
    \end{subfigure}
    \begin{subfigure}[t]{0.245\linewidth}
    	\includegraphics[width=\linewidth]{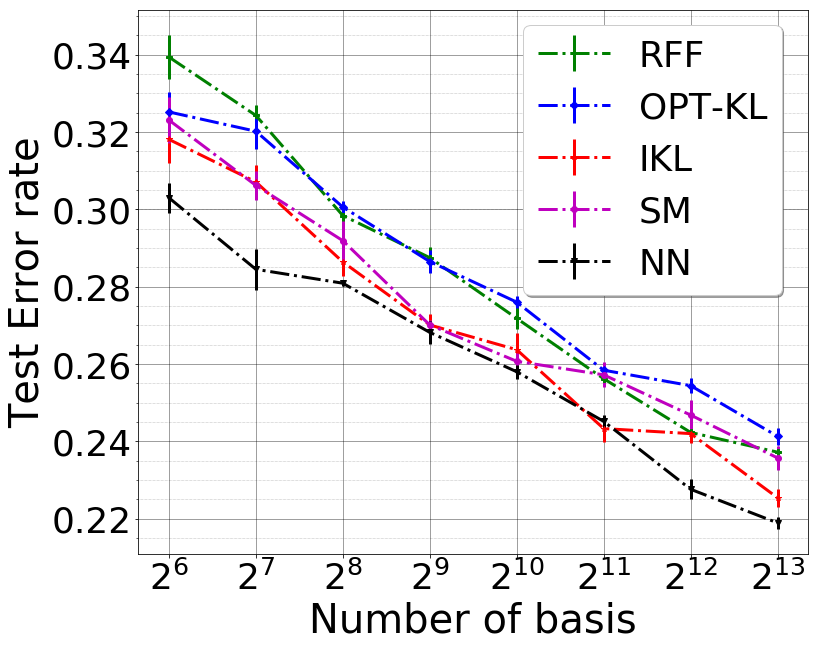}
        \caption{CIFAR-10 (auto-truck)}
    \end{subfigure}
    \hfill%
    \begin{subfigure}[t]{0.245\linewidth}
    	\includegraphics[width=\linewidth]{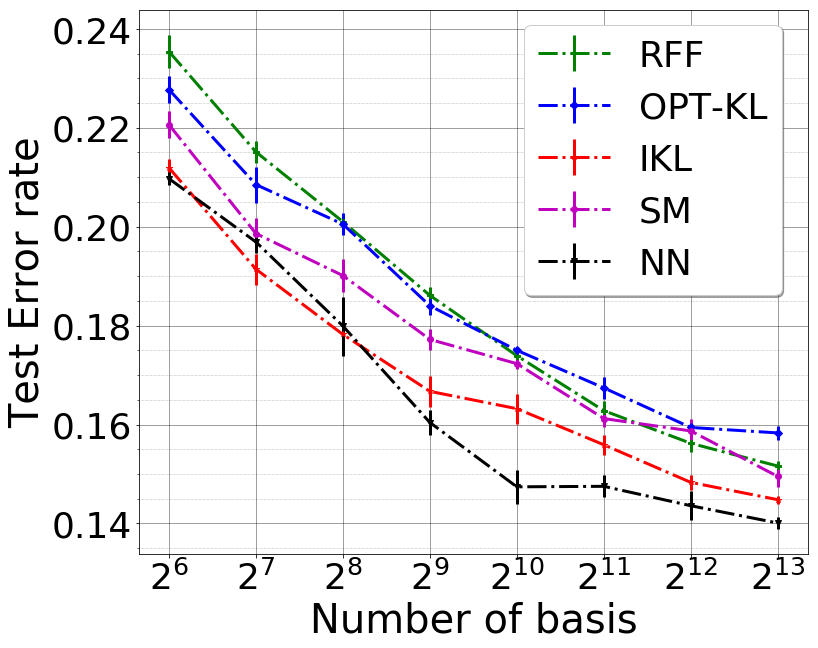}
        \caption{CIFAR-10 (plane-bird)}
    \end{subfigure}

    \caption{Test error rate versus number of basis in second stage on benchmark binary classification tasks.
	We report mean and standard deviation over five runs.
    Our method (IKL) is compared with RFF~\citep{rahimi2009weighted}, OPT-KL~\citep{sinha2016learning}, SM~\citep{wilson2013gaussian} and the end-to-end training MLP (NN).}
	\label{fig:exp-spv-real}
\end{figure*}

\subsection{Consistency and Generalization}
The simple two-stages approach, IKL with RKS, allows us to provide the consistency and generalization
guarantees. For consistency, it guarantees the solution of finite sample approximations of ~\eqref{eq:k_align}
approach to the optimum of ~\eqref{eq:k_align} (population optimum), 
when we increase number of training data and number of random features. We firstly define necessary symbols and state the theorem. 

Let $s(x_i, x_j)=y_iy_j$ be a label similarity function, where $|y_i|\leq 1$. We use $s_{ij}$ to denote $s(x_i, x_j)$
interchangeably. Given a kernel $k$, we define the true and empirical alignment functions as,
\[
\begin{array}{ccl}
	T(k) & = & \EE\left[ s(x, x')k(x, x') \right] \\
	\hat{T}(k) & = & \frac{1}{n(n-1)}\sum_{i\neq j}s_{ij}k(x_i, x_j). \\
\end{array}
\]

In the following, we abuse the notation $k_\psi$ to be $k_h$ for ease of illustration.
Recall the definitions of $k_h(x,x')=\langle\phi_h(x), \phi_h(x')  \rangle$ and
$\hat{k}_h(x, x')=\hat{\phi}_h(x)^\top\hat{\phi}_h(x')$.
We define two hypothesis sets
\[
	\begin{array}{l}
		\displaystyle \Fcal_\Hcal  = \{ f(x) = \langle w, \phi_h(x) \rangle_H | h\in \Hcal, \langle w, w\rangle \leq 1 \} \\
		\displaystyle \hat{\Fcal}_\Hcal^m  =  \{ f(x) = w^\top\hat{\phi}_h(x)  | h\in \Hcal, \|w\| \leq 1, w\in \RR^m \}. 
	\end{array}
\]
\begin{definition}(Rademacher's Complexity)
Given a hypothesis set $\Fcal$, where $f: \Xcal\times \Xcal \rightarrow \RR$ if $f\in\Fcal$, and a fixed sample $X = \{x_1, \dots, x_n\}$, 
the empirical Rademacher's complexity of $\Fcal$ is defined as 
\[
	\mathfrak{R}_X^n(\Fcal) = \frac{1}{n} \EE_\sigma\left[ \sup_{f\in \Fcal} \sum_{i=1}^n \sigma_i f(x_i) \right],
\]
where $\sigma$ are $n$ i.i.d. Rademacher random variables.
\end{definition}

We then have the following theorems showing that the consistency guarantee 
depends on the complexity of the function class induced by IKL as well as the number of random features.
The proof can be found in Appendix~\ref{sec:pf_consist}.
\begin{theorem} (Consistency)
Let $\hat{h} = \arg\max_{h\in\Hcal} \hat{T}(\hat{k}_h)$, with i.i.d, samples $\{\nu_i\}_{i=1}^m$ drawn from
$\PP(\nu)$. With probability at least $1-3\delta$, we have  $|T(\hat{k}_{\hat{h}}) - \sup_{h\in \Hcal} T(k_h) | \leq $
\[
\begin{array}{c}
\displaystyle	2\EE_X\left[ \mathfrak{R}_X^{n-1}(\Fcal_\Hcal) + \mathfrak{R}_X^{n-1}(\hat{\Fcal}_\Hcal^m) \right]+
\sqrt{\frac{8\log\frac{1}{\delta}}{n}} + \sqrt{\frac{2\log\frac{4}{\delta}}{m}}. 
\end{array}
\]
\label{thm:consist}
\end{theorem}

Applying~\cite{cortes2010generalization},  
We also have a generalization bound, which depends number of training data $n$,
number of random features $m$ and the Rademacher complexity of IKL kernel, as shown in Appendix~\ref{sec:general}.
The Rademacher complexity $\mathfrak{R}_X^{n}(\Fcal_\Hcal)$, for example, can be $1/\sqrt{n}$ or even
$1/n$ for kernels with different bounding conditions~\citep{cortes2013learning}. 
We would expect worse rates for more powerful kernels.
It suggests the trade-off between consistency/generalization and using powerful kernels parametrized by neural networks.

\section{Discussion}
\label{sec:discussion}

We propose a generic kernel learning algorithm, IKL, which
learns sampling processes of kernel spectral distributions by transforming samples from a base distribution $\PP(\nu)$
into ones for the other kernel (spectral density). 
We compare IKL with other algorithms for learning MMD GAN and supervised learning with Random Kitchen Sinks (RKS). 
For these two tasks, the conditions and guarantees of IKL for are studied.
Empirical studies show IKL is better than or competitive with the state-of-the-art kernel learning algorithms.
It proves IKL can learn to transform $\PP(\nu)$ into effective kernels even if $\PP(\nu)$ is less 
less favorable to the task.

We note that the preliminary idea of IKL is mentioned in~\cite{buazuavan2012fourier}, but they ended up with a
algorithm that directly optimizes sampled random features (RF), which has many follow-up works (\eg~\citet{sinha2016learning,
bullins2017not}). The major difference is, by learning the transformation function~$h_\psi$, the RF used in
training and evaluation can be different. This flexibility allows a simple training algorithm (SGD) and does not require
to keep learned features. In our studies on GAN and RKS, we show using a simple MLP can already achieve better
or competitive
performance with those works, which suggest IKL can be a new direction for kernel learning and worth more studies.

We highlight that IKL is not conflict with existing works but can be combined with them. 
In Section~\ref{sec:gan}, we show combining IKL with kernel learning via embedding~\citep{wilson2016deep} and mixture
of spectral distributions~\citep{wilson2013gaussian}. 
Therefore, 
in addition to the examples shown in Section~\ref{sec:gan} and Section~\ref{sec:classification},
IKL is directly applicable to many existing works with kernel learning via embedding (\eg~\cite{dai2014scalable,
li2016utilize, wilson2016deep, al2016learning, arbel2018gradient, jean2018semi, chang2019kernel}). 
A possible extension is combining with 
Bayesian inference~\citep{oliva2016bayesian} under the framework similar to \citet{saatchi2017bayesian}.
The learned sampler from IKL can possibly provide an easier way to
do Bayesian inference via sampling.

\bibliography{main}

\begin{thebibliography}{}

\bibitem[Al-Shedivat et~al., 2016]{al2016learning}
Al-Shedivat, M., Wilson, A.~G., Saatchi, Y., Hu, Z., and Xing, E.~P. (2016).
\newblock Learning scalable deep kernels with recurrent structure.
\newblock {\em arXiv preprint arXiv:1610.08936}.

\bibitem[Alber et~al., 2017]{alber2017empirical}
Alber, M., Kindermans, P.-J., Sch{\"u}tt, K., M{\"u}ller, K.-R., and Sha, F.
  (2017).
\newblock An empirical study on the properties of random bases for kernel
  methods.
\newblock In {\em NIPS}.

\bibitem[Arbel et~al., 2018]{arbel2018gradient}
Arbel, M., Sutherland, D.~J., Bi{\'n}kowski, M., and Gretton, A. (2018).
\newblock On gradient regularizers for mmd gans.
\newblock In {\em NIPS}.

\bibitem[Arjovsky and Bottou, 2017]{arjovsky2017towards}
Arjovsky, M. and Bottou, L. (2017).
\newblock Towards principled methods for training generative adversarial
  networks.
\newblock In {\em ICLR}.

\bibitem[Arjovsky et~al., 2017]{ArjovskyCB17wgan}
Arjovsky, M., Chintala, S., and Bottou, L. (2017).
\newblock Wasserstein {GAN}.
\newblock In {\em ICML}.

\bibitem[Bach, 2009]{bach2009exploring}
Bach, F.~R. (2009).
\newblock Exploring large feature spaces with hierarchical multiple kernel
  learning.
\newblock In {\em NIPS}.

\bibitem[Bach et~al., 2004]{bach2004multiple}
Bach, F.~R., Lanckriet, G.~R., and Jordan, M.~I. (2004).
\newblock Multiple kernel learning, conic duality, and the smo algorithm.
\newblock In {\em ICML}.

\bibitem[B{\u{a}}z{\u{a}}van et~al., 2012]{buazuavan2012fourier}
B{\u{a}}z{\u{a}}van, E.~G., Li, F., and Sminchisescu, C. (2012).
\newblock Fourier kernel learning.
\newblock In {\em ECCV}.

\bibitem[Bellemare et~al., 2017]{bellemare2017cramer}
Bellemare, M.~G., Danihelka, I., Dabney, W., Mohamed, S., Lakshminarayanan, B.,
  Hoyer, S., and Munos, R. (2017).
\newblock The cramer distance as a solution to biased wasserstein gradients.
\newblock {\em arXiv preprint arXiv:1705.10743}.

\bibitem[Bi{\'n}kowski et~al., 2018]{binkowski2018demystifying}
Bi{\'n}kowski, M., Sutherland, D.~J., Arbel, M., and Gretton, A. (2018).
\newblock Demystifying mmd gans.
\newblock In {\em ICLR}.

\bibitem[Borisenko and Minchenko, 1992]{borisenko1992directional}
Borisenko, O. and Minchenko, L. (1992).
\newblock Directional derivatives of the maximum function.
\newblock {\em Cybernetics and Systems Analysis}, 28(2):309--312.

\bibitem[Bullins et~al., 2018]{bullins2017not}
Bullins, B., Zhang, C., and Zhang, Y. (2018).
\newblock Not-so-random features.
\newblock In {\em ICLR}.

\bibitem[Chang et~al., 2017]{chang2017data}
Chang, W.-C., Li, C.-L., Yang, Y., and Poczos, B. (2017).
\newblock Data-driven random fourier features using stein effect.
\newblock In {\em IJCAI}.

\bibitem[Chang et~al., 2019]{chang2019kernel}
Chang, W.-C., Li, C.-L., Yang, Y., and P{\'o}czos, B. (2019).
\newblock Kernel change-point detection with auxiliary deep generative models.
\newblock {\em arXiv preprint arXiv:1901.06077}.

\bibitem[Cortes et~al., 2013]{cortes2013learning}
Cortes, C., Kloft, M., and Mohri, M. (2013).
\newblock Learning kernels using local rademacher complexity.
\newblock In {\em NIPS}.

\bibitem[Cortes et~al., 2010]{cortes2010generalization}
Cortes, C., Mohri, M., and Rostamizadeh, A. (2010).
\newblock Generalization bounds for learning kernels.
\newblock In {\em ICML}.

\bibitem[Cristianini et~al., 2002]{cristianini2002kernel}
Cristianini, N., Shawe-Taylor, J., Elisseeff, A., and Kandola, J.~S. (2002).
\newblock On kernel-target alignment.
\newblock In {\em ICML}.

\bibitem[Dai et~al., 2014]{dai2014scalable}
Dai, B., Xie, B., He, N., Liang, Y., Raj, A., Balcan, M.-F.~F., and Song, L.
  (2014).
\newblock Scalable kernel methods via doubly stochastic gradients.
\newblock In {\em NIPS}.

\bibitem[Dudley, 2018]{dudley2018real}
Dudley, R.~M. (2018).
\newblock {\em Real Analysis and Probability}.
\newblock Chapman and Hall/CRC.

\bibitem[Duvenaud et~al., 2013]{duvenaud2013structure}
Duvenaud, D., Lloyd, J.~R., Grosse, R., Tenenbaum, J.~B., and Ghahramani, Z.
  (2013).
\newblock Structure discovery in nonparametric regression through compositional
  kernel search.
\newblock In {\em ICML}.

\bibitem[Dziugaite et~al., 2015]{DziugaiteRG15}
Dziugaite, G.~K., Roy, D.~M., and Ghahramani, Z. (2015).
\newblock Training generative neural networks via maximum mean discrepancy
  optimization.
\newblock In {\em UAI}.

\bibitem[Fan et~al., 2008]{fan2008liblinear}
Fan, R.-E., Chang, K.-W., Hsieh, C.-J., Wang, X.-R., and Lin, C.-J. (2008).
\newblock Liblinear: A library for large linear classification.
\newblock {\em JMLR}.

\bibitem[G{\"o}nen and Alpayd{\i}n, 2011]{gonen2011multiple}
G{\"o}nen, M. and Alpayd{\i}n, E. (2011).
\newblock Multiple kernel learning algorithms.
\newblock {\em JMLR}.

\bibitem[Goodfellow et~al., 2014]{Goodfellow14GAN}
Goodfellow, I.~J., Pouget{-}Abadie, J., Mirza, M., Xu, B., Warde{-}Farley, D.,
  Ozair, S., Courville, A.~C., and Bengio, Y. (2014).
\newblock Generative adversarial nets.
\newblock In {\em NIPS}.

\bibitem[Gretton et~al., 2012]{Gretton2012ktest}
Gretton, A., Borgwardt, K.~M., Rasch, M.~J., Sch\"{o}lkopf, B., and Smola, A.
  (2012).
\newblock A kernel two-sample test.
\newblock {\em JMLR}.

\bibitem[Gulrajani et~al., 2017]{gulrajani2017improved}
Gulrajani, I., Ahmed, F., Arjovsky, M., Dumoulin, V., and Courville, A. (2017).
\newblock Improved training of wasserstein gans.
\newblock In {\em NIPS}.

\bibitem[He et~al., 2016]{he2016deep}
He, K., Zhang, X., Ren, S., and Sun, J. (2016).
\newblock Deep residual learning for image recognition.
\newblock In {\em CVPR}.

\bibitem[Heusel et~al., 2017]{heusel2017gans}
Heusel, M., Ramsauer, H., Unterthiner, T., Nessler, B., and Hochreiter, S.
  (2017).
\newblock Gans trained by a two time-scale update rule converge to a local nash
  equilibrium.
\newblock In {\em NIPS}.

\bibitem[Hinton and Salakhutdinov, 2008]{hinton2008using}
Hinton, G.~E. and Salakhutdinov, R.~R. (2008).
\newblock Using deep belief nets to learn covariance kernels for gaussian
  processes.
\newblock In {\em NIPS}.

\bibitem[Jean et~al., 2018]{jean2018semi}
Jean, N., Xie, S.~M., and Ermon, S. (2018).
\newblock Semi-supervised deep kernel learning: Regression with unlabeled data
  by minimizing predictive variance.
\newblock In {\em Advances in Neural Information Processing Systems}, pages
  5327--5338.

\bibitem[Krizhevsky and Hinton, 2009]{krizhevsky2009learning}
Krizhevsky, A. and Hinton, G. (2009).
\newblock Learning multiple layers of features from tiny images.

\bibitem[Krizhevsky et~al., 2012]{krizhevsky2012imagenet}
Krizhevsky, A., Sutskever, I., and Hinton, G.~E. (2012).
\newblock Imagenet classification with deep convolutional neural networks.
\newblock In {\em NIPS}.

\bibitem[Lanckriet et~al., 2004]{lanckriet2004learning}
Lanckriet, G.~R., Cristianini, N., Bartlett, P., Ghaoui, L.~E., and Jordan,
  M.~I. (2004).
\newblock Learning the kernel matrix with semidefinite programming.
\newblock {\em JMLR}.

\bibitem[LeCun et~al., 1998]{lecun1998gradient}
LeCun, Y., Bottou, L., Bengio, Y., and Haffner, P. (1998).
\newblock Gradient-based learning applied to document recognition.
\newblock {\em Proceedings of the IEEE}.

\bibitem[Li et~al., 2017]{Li2017mmdgan}
Li, C.-L., Chang, W.-C., Cheng, Y., Yang, Y., and Poczos, B. (2017).
\newblock Mmd gan: Towards deeper understanding of moment matching network.
\newblock In {\em NIPS}.

\bibitem[Li and P{\'o}czos, 2016]{li2016utilize}
Li, C.-L. and P{\'o}czos, B. (2016).
\newblock Utilize old coordinates: Faster doubly stochastic gradients for
  kernel methods.
\newblock In {\em UAI}.

\bibitem[Li et~al., 2015]{Li2015GMM}
Li, Y., Swersky, K., and Zemel, R. (2015).
\newblock Generative moment matching networks.
\newblock In {\em ICML}.

\bibitem[MacKay, 1995]{mackay1995bayesian}
MacKay, D.~J. (1995).
\newblock Bayesian neural networks and density networks.
\newblock {\em Nuclear Instruments and Methods in Physics Research Section A:
  Accelerators, Spectrometers, Detectors and Associated Equipment}.

\bibitem[Mairal, 2016]{mairal2016end}
Mairal, J. (2016).
\newblock End-to-end kernel learning with supervised convolutional kernel
  networks.
\newblock In {\em NIPS}.

\bibitem[Mairal et~al., 2014]{mairal2014convolutional}
Mairal, J., Koniusz, P., Harchaoui, Z., and Schmid, C. (2014).
\newblock Convolutional kernel networks.
\newblock In {\em NIPS}.

\bibitem[Mroueh et~al., 2018]{mroueh2017sobolev}
Mroueh, Y., Li, C.-L., Sercu, T., Raj, A., and Cheng, Y. (2018).
\newblock Sobolev gan.
\newblock In {\em ICLR}.

\bibitem[Mroueh and Sercu, 2017]{mroueh2017fisher}
Mroueh, Y. and Sercu, T. (2017).
\newblock Fisher gan.
\newblock In {\em NIPS}.

\bibitem[Mroueh et~al., 2017]{Mroueh2017mcgan}
Mroueh, Y., Sercu, T., and Goel, V. (2017).
\newblock Mcgan: Mean and covariance feature matching gan.
\newblock In {\em ICML}.

\bibitem[Nowozin et~al., 2016]{NowozinCT16fgan}
Nowozin, S., Cseke, B., and Tomioka, R. (2016).
\newblock f-gan: Training generative neural samplers using variational
  divergence minimization.
\newblock In {\em NIPS}.

\bibitem[Oliva et~al., 2016]{oliva2016bayesian}
Oliva, J.~B., Dubey, A., Wilson, A.~G., P{\'o}czos, B., Schneider, J., and
  Xing, E.~P. (2016).
\newblock Bayesian nonparametric kernel-learning.
\newblock In {\em AISTATS}.

\bibitem[Radford et~al., 2016]{radford2015unsupervised}
Radford, A., Metz, L., and Chintala, S. (2016).
\newblock Unsupervised representation learning with deep convolutional
  generative adversarial networks.
\newblock In {\em ICLR}.

\bibitem[Rahimi and Recht, 2007]{Rahimi_NIPS_07}
Rahimi, A. and Recht, B. (2007).
\newblock Random features for large-scale kernel machines.
\newblock In {\em NIPS}.

\bibitem[Rahimi and Recht, 2009]{rahimi2009weighted}
Rahimi, A. and Recht, B. (2009).
\newblock Weighted sums of random kitchen sinks: Replacing minimization with
  randomization in learning.
\newblock In {\em NIPS}.

\bibitem[Rudin, 2011]{Rudin_book_11}
Rudin, W. (2011).
\newblock {\em Fourier analysis on groups}.
\newblock John Wiley \& Sons.

\bibitem[Saatchi and Wilson, 2017]{saatchi2017bayesian}
Saatchi, Y. and Wilson, A.~G. (2017).
\newblock Bayesian gan.
\newblock In {\em NIPS}, pages 3625--3634.

\bibitem[Salimans et~al., 2016]{salimans2016improved}
Salimans, T., Goodfellow, I., Zaremba, W., Cheung, V., Radford, A., and Chen,
  X. (2016).
\newblock Improved techniques for training gans.
\newblock In {\em NIPS}.

\bibitem[Sinha and Duchi, 2016]{sinha2016learning}
Sinha, A. and Duchi, J.~C. (2016).
\newblock Learning kernels with random features.
\newblock In {\em NIPS}.

\bibitem[Wilson and Adams, 2013]{wilson2013gaussian}
Wilson, A. and Adams, R. (2013).
\newblock Gaussian process kernels for pattern discovery and extrapolation.
\newblock In {\em ICML}.

\bibitem[Wilson et~al., 2016]{wilson2016deep}
Wilson, A.~G., Hu, Z., Salakhutdinov, R., and Xing, E.~P. (2016).
\newblock Deep kernel learning.
\newblock In {\em AISTATS}.

\bibitem[Yang et~al., 2015]{yang2015carte}
Yang, Z., Wilson, A., Smola, A., and Song, L. (2015).
\newblock A la carte--learning fast kernels.
\newblock In {\em AISTATS}.

\bibitem[Zhang et~al., 2017]{zhang2017hitting}
Zhang, Y., Liang, P., and Charikar, M. (2017).
\newblock A hitting time analysis of stochastic gradient langevin dynamics.
\newblock In {\em COLT}.

\end{thebibliography}
\bibliographystyle{apalike}

\newpage
\appendix
\onecolumn

\section{Proof of Theorem~\ref{thm:cont}}

We first show $\PP_n \xrightarrow[]{D} \PP$ then $\max_{\psi, \varphi}M_{\psi, \varphi}(\PP, \PP_n) \rightarrow 0$.
The results are based on~\citet{arbel2018gradient}, which leverages Corollary 11.3.4 of~\citet{dudley2018real}. 
Follows the sketch of \citet{arbel2018gradient}, the only thing we remain to show is proving $\|k_\psi(f_\varphi(x), \cdot)-k_\psi(f_\varphi(y),
\cdot)\|_{\Hcal_K}$ is Lipschitz. By definition, we know that
$\|k_\psi(f_\varphi(x), \cdot)-k_\psi(f_\varphi(y),\cdot)\|_{\Hcal_K} = 2(1- k_\psi(f_\varphi(x), f_\varphi(y)) )$.
Also, since $k_\psi(0)=1$ and $k_\psi(0) - k_\psi(x,x') \leq L_k\| 0-(x-x') \|$ (Lipschtiz assumption of $k_\psi$),
we have 
\[
	\|k_\psi(f_\varphi(x), \cdot)-k_\psi(f_\varphi(y),\cdot)\|_{\Hcal_K} \leq 2L_k\|f_\varphi(x)- f_\varphi(y)\| \leq 
	2L_k L \|x-y\|,
\]
where the last inequality is since $f_\varphi$ is also a Lipschitz function with Lipschitz constant $L$.

The other direction, $\max_{\psi, \varphi}M_{\psi, \varphi}(\PP, \PP_n) \rightarrow 0$ then $\PP_n \xrightarrow[]{D} \PP$, 
is relatively simple. Without loss of generality, we assume there exists $\psi'$ and $\phi'$ such  that $h_\psi$ and $f_\phi$ are identity
functions (up to scaling), which recover the Gaussian kernel $k$. Therefore, 
$\max_{\psi, \varphi}M_{\psi, \varphi}(\PP, \PP_n) \rightarrow 0$ implies $M_{\psi', \varphi'}(\PP,\PP_n) \rightarrow
0$, which completes the proof because MMD with any Gaussian kernel is weak~\citep{Gretton2012ktest}.

\subsection{Continuity}
\begin{lemma}(\citet{borisenko1992directional})
Define $\tau(x) = \max\{f(x, u) | u\in U\}$. If $f$ is locally Lipschitz in $x$, $U$ is compact and $\bigtriangledown f(x, u^*(x))$
exists, where $u^*(x) =  \arg\max_u f(x, u)$, then $\tau(x)$ is differentiable almost everywhere. 
\label{lem:max}
\end{lemma}

We are going to show 
\begin{equation}
	\max_{\psi,\varphi} M_{\psi,\varphi}(\PP_\Xcal, \PP_\theta) = \EE_{x,x'}[k_{\psi,\varphi}(x, x')] -2\EE_{x,z}[k_{\psi,\varphi}(x,g_\theta(z))]  + \EE_{z,z'}[ k_{\psi,\varphi}(g_\theta(z'), g_\theta(z))  ] 
	\label{eq:mmd_obj}
\end{equation}
is differentiable with respect to $\varphi$ almost everywhere by using the auxiliary Lemma~\ref{lem:max}. 
We fist show $\EE_{z,z'}[ k_{\psi,\varphi}(g_\theta(z'), g_\theta(z))  ]$ in~\eqref{eq:mmd_obj} is locally Lipschitz in
$\theta$. By definition, $k_{\psi,\varphi}(x, x') = k_\psi(f_\varphi(x )-f_\varphi( x' ) )$, therefore,
\[
\begin{aligned}
    & \displaystyle \EE_{x,x'} \bigg[ k_{\psi,\varphi}\Big( g_\theta(z), g_\theta(z') \Big) - k_{\psi,\varphi}\Big( g_{\theta'}(z), g_{\theta'}(z')  \Big) \bigg] \\
	= & \displaystyle \EE_{z,z'}\bigg[ k_\psi\Big( f_\varphi\big( g_\theta(z) \big) - f_\varphi\big( g_\theta(z') \big)\Big) \bigg] - 
	  \EE_{z,z'}\bigg[ k_\psi\Big( f_\varphi\big( g_{\theta'}(z) \big) - f_\varphi\big( g_{\theta'}(z') \big)\Big) \bigg] \\
   \leq & \displaystyle \EE_{z,z'}\bigg[ L_k \Big\Vert f_\varphi\big( g_\theta(z) \big) - f_\varphi\big( g_\theta(z') \big) - f_\varphi\big( g_{\theta'}(z) \big) + f_\varphi\big( g_{\theta'}(z') \big) \Big\Vert \bigg] \\
   \leq & \displaystyle \EE_{z,z'}\Big[ L_k L(\theta, z) \| \theta-\theta'  \| + L_k L(\theta, z') \| \theta-\theta'  \| \Big] \\
   = & 2L_k\EE_{z}\big[  L(\theta, z) \big]\| \theta-\theta'  \|. \\
\end{aligned}
\]
The first inequality is followed by the assumption that $k$ is locally Lipschitz in $(x,x')$, with a upper bound
$L_k$ for Lipschitz constants.
By Assumption~\ref{ass:cont}, $\EE_{z}\big[  L(\theta, z) \big]<\infty$,
we prove $\EE_{z,z'}\Big[ k_\psi \big( f_\varphi( g_\theta(z) ) - f_\varphi( g_\theta(z') ) \big) \Big]$ is locally Lipschitz. 
The similar argument is applicable to other terms in~\eqref{eq:mmd_obj}; therefore,~\eqref{eq:mmd_obj} is locally
Lipschitz in $\theta$.  

Last, with the compactness assumption on $\Phi$ and $\Psi$, and differentiable assumption on
$M_{\psi,\varphi}(\PP_\Xcal, \PP_\theta)$, applying Lemma~\ref{lem:max} proves Theorem~\ref{thm:cont}.

\section{Proof of Lemma~\ref{lem:var}}
Without loss of the generality, we can rewrite the kernel function as
$k_\psi(t) = \EE_\nu \Big[ \cos\big( h_\psi(\nu)^\top t \big) \Big]$,
where $t$ is bounded. We then have
\[
	\begin{array}{ccl}
	\|\nabla_t k_\psi(t)\| & = & \Big\Vert \EE_\nu\Big[ \sin\big( h_\psi(\nu)^\top t \big) h_\psi(\nu) \Big] \Big\Vert \\
	& \leq & \EE_\nu\Big[ \big\vert \sin(h_\psi(\nu)^\top t) \big\vert \times \|h_\psi(\nu)\| \Big] \\
	& \leq & \EE_\nu\big[ \|t\| \|h_\psi(\nu)\|^2 \big]
	\end{array}
\]
The last inequality follows by $|\sin(x)| < |x|$. Since $t$ is bounded, if $\EE_\nu[\|h_\psi(\nu)\|^2]<\infty$, there exist a
constant $L$ such that $\|\nabla_t k_\psi(t)\| < L, \forall t$. 

By mean value theorem, for any $t$ and $t'$, there exists $s=\alpha t + (1-\alpha)t'$, where $\alpha \in [0,1]$, such that
\[
	k_\psi(t) - k_\psi(t') = \nabla_{s} k_\psi( s )^\top (t-t'). 
\]
Combining with $\|\nabla_t k_\psi(t)\| < L, \forall t$, we prove 
\[
	k_\psi(t) - k_\psi(t')  \leq L\|t-t'\|.
\]

\section{Additional Studies of MMD GAN with IKL}
\label{sec:addition}
\subsection{Additional Quantitative Results}
We show the full quantitative results on MMD GANs with different kernels with mean and standard error 
in Table~\ref{tb:gan_full_results}.
In every tasks, IKL is the best among the predefined base kernels (Gaussian, RQ) and the competitive kernel learning
algorithm (SM).
The difference in FID is less significant than inception score and JS-4, but we note that FID score is a biased evaluation metric as discussed in~\citet{binkowski2018demystifying}.

\label{sec:addition}
\begin{table}[h]
	\centering
    \begin{tabular}{c|c|c|c}
        \toprule
        Method 		& Inception Scores $(\uparrow)$ & FID Scores $(\downarrow)$	& JS-4 $(\downarrow)$ \\
        \midrule
        Gaussian 	& $6.726 \pm 0.021$  			& $32.50 \pm 0.07$			& $0.381 \pm 0.003$\\
        RQ 			& $6.785 \pm 0.031$ 			& $32.20 \pm 0.09$			& $0.463 \pm 0.005$\\
		SM 			& $6.746 \pm 0.031$ 			& $32.43 \pm 0.08$						& $0.378 \pm 0.003$ \\
        IKL 		& $\mathbf{6.876 \pm 0.018}$ 	& $\mathbf{31.98 \pm 0.05}$ & $\mathbf{0.372 \pm 0.002}$\\
        \midrule
        WGAN-GP 	& $6.539 \pm 0.034$  			& $36.413 \pm 0.05$			& $0.379 \pm 0.002$ \\
        \bottomrule
    \end{tabular}
    \caption{Inception scores, FID scores, and JS-4 divergece results.}
    \label{tb:gan_full_results}
\end{table}
\subsection{Computational Issues of GAN trainings with IKL}
\label{sec:comp}
\paragraph{Model Capacity}
For $f_\varphi$, the number of parameters for DCGAN is around $0.8$ million for size $32\times 32$ images and $3$ millions
for size $64\times 64$ images. The ResNet architecture used in~\cite{gulrajani2017improved} has around $10$ millions
parameters.
In contrast, in all experiments, we use simple three layer MLP as $h_\psi$ for IKL, where the input and output dimensions are 16, and hidden
layer size is 32. The total parameters are just around 2,000. Compared with $f_\phi$, the additional number of
parameters used for $h_\psi$ is almost negligible.

\begin{wrapfigure}{r}{.5\textwidth}
      \includegraphics[width=\linewidth]{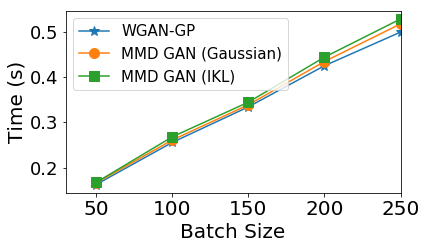}
\end{wrapfigure}
\paragraph{Computational Time}
The potential concern of IKL is sampling random features for each examples. In our experiments, we use
$m=1024$ random features for each iteration. 
We measure the time per iteration of updating critic iterations ($f$ for WGAN-GP and MMD GAN with Gaussian kernel; $f$
and $h$ for IKL) with different batch sizes under Titan X.
The difference between WGAN-GP, MMD GAN and IKL
are not significant.  The reason is computing MMD and random feature is highly parallelizable, and other computation, such as
evaluating $f_\phi$ and its gradient penalty, dominates the cost because $f_\phi$ has much more parameters as
aforementioned. Therefore, we believe the proposed IKL is still cost effective in practice.

\subsection{Detailed Discussion of Variance Constraints}
In Section~\ref{sec:gan}, we propose to constrain variance via  $\lambda_h(\EE_\nu\left[\|h_\psi(\nu)\|^2\right] -u
)^2$. There are other alternatives, such as constraining $L^2$ penalty or using Langrange. In practice, we do not
observe significant difference. 

Although we show the necessity of the variance constraint in language generation in Figure~\ref{fig:novar},
we remark that the proposed constraint is a sufficient condition. 
For CIFAR-10, without the constraint, we observe that the variance is still bouncing between $1$ and $2$
without explosion as Figure~\ref{fig:novar}.
Therefore, the training leads to a satisfactory result with $6.731 \pm 0.034$ inception score, but it is slightly 
worse than IKL in Table~\ref{tb:gan_results}.
The necessary or weaker sufficient conditions are worth further studying as a future work.

\subsection{IKL with and without Neural Networks on GAN training}
\label{sec:no_h}
Instead of learning a transform function $h_\psi$ for the spectral distribution
as we proposed in Section~\ref{sec:kernel} (IKL-NN), 
the other realization of IKL is to keep a pool of finite number learned random features  $\Omega = \{ \hat{\omega_i}\}_{i=1}^m$, and
approximate the kernel evaluation by $\hat{k}_\Omega(x, x')= \hat{\phi}_{\Omega}(x)^\top \hat{\phi}_{\Omega}(x')$,
where $\hat{\phi}_{\Omega}(x)^\top = [\phi(x; \hat{\omega}_1), \dots, \phi(x; \hat{\omega}_m)]$.
During the learning, it directly optimize $\hat{\omega_i}$.
Many existing works study this idea for supervised learning, such as \citet{buazuavan2012fourier, yang2015carte, sinha2016learning,
chang2017data, bullins2017not}.
We call the latter realization as IKL-RFF. Next, we discuss
and compare the difference between IKL-NN and IKL-RFF.

The crucial difference between IKL-NN and IKL-RFF is, IKL-NN can sample arbitrary number of random features
by first sampling $\nu \sim \PP(\nu)$ and transforming it via $h_{\psi}(\nu)$, while IKL-RFF is restricted by the pool size
$m$. If the application needs more random features, IKL-RFF will be memory inefficient. 
Specifically, we compare IKL-NN and IKL-RFF with different number of random features in Figure \ref{fig:ikl_nn_rff}.
With the same number of parameters (i.e., $|h_\psi| = m \times dim(\nu)$)\footnote{$|h_\psi|$ denotes
number of parameters in $h_\psi$, $m$ is number of random features and $dim(\nu)$ is the dimension of the $\nu$.}
, IKL-NN outperforms IKL-RFF of $m=128$ on Inception scores ($6.876$ versus $6.801$).
For IKL-RFF to achieve the same or better Inception scores of IKL-NN, the number of random features
$m$ needs increasing to $4096$, which is less memory efficient than the IKL-NN realization.
In particular, $h_\psi$ of IKL-NN is a three-layers MLP with $2048$ number of parameters
($16\times32 + 32\times32 + 32\times16$),
while IKL-RFF has $2048, 65536$ number of parameters, for $m=128,4096$, respectively.

\begin{table}[ht]
  \begin{minipage}[b]{0.5\linewidth}
    \centering
    \includegraphics[width=\textwidth]{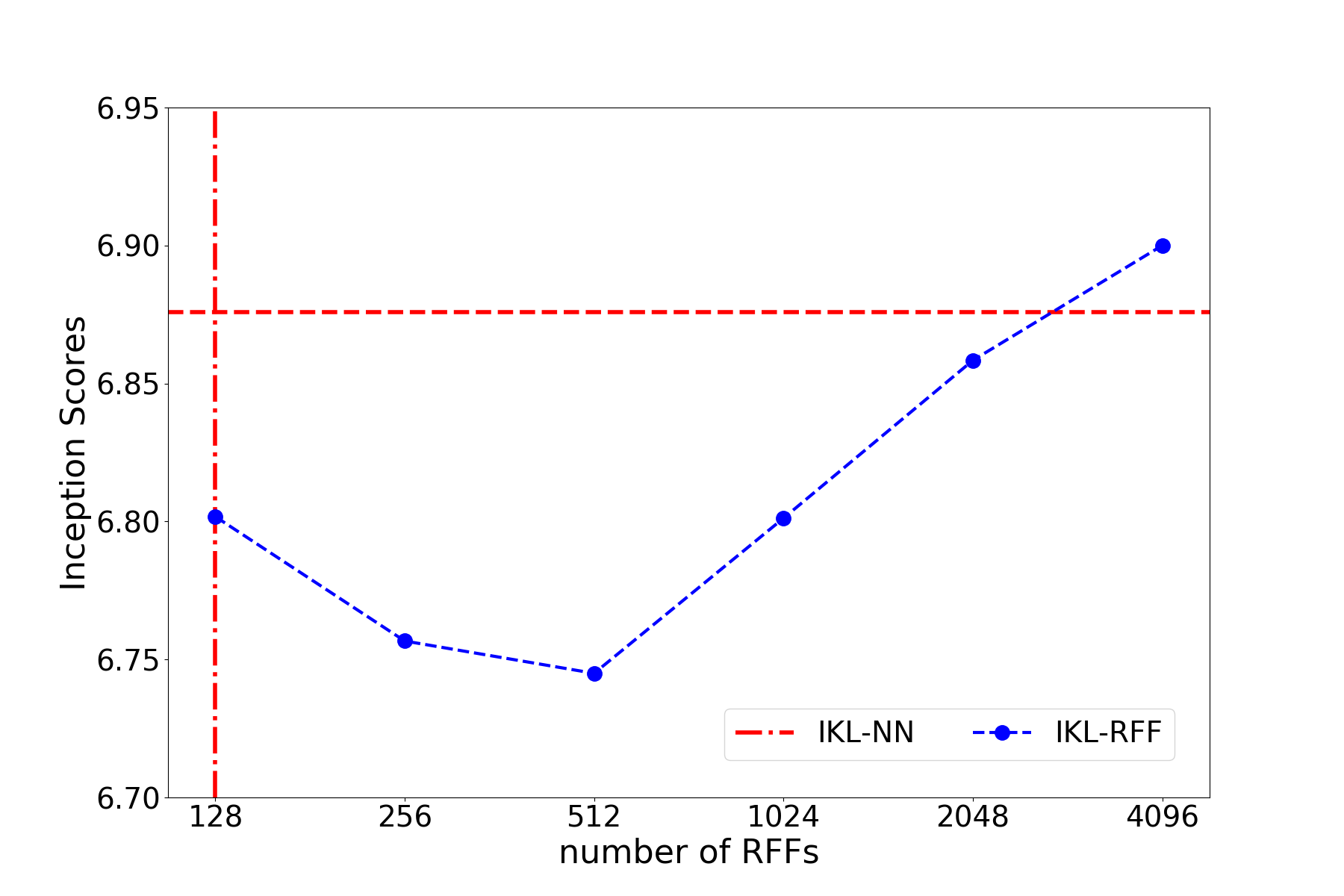}
    \captionof{figure}{The comparison between IKL-NN and IKL-RFF on CIFAR-10 under
    different number of random features.}
    \label{fig:ikl_nn_rff}
  \end{minipage}
  \hfill
  \begin{varwidth}[b]{0.45\linewidth}
    \centering
    \begin{tabular}{ l r r r }
      \toprule
      Algorithm & JS-4 \\
      \midrule
      IKL-NN & $0.372 \pm 0.002$ \\
      IKL-RFF & $0.383 \pm 0.002$ \\
      \midrule
	  IKL-RFF (+2) & $0.380 \pm 0.002$ \\
      IKL-RFF (+4) & $0.377 \pm 0.002$ \\
      IKL-RFF (+8) & $0.375 \pm 0.002$ \\
      \bottomrule
    \end{tabular}
    \vspace{1em}
    \caption{ The comparison between IKL-NN and IKL-RFF on Google Billion Word. }
    \label{tb:ikl_js4}
  \end{varwidth}%
\end{table}

On the other hand, using large $m$ for IKL-RFF not only increases the number of parameters, but might also enhance the
optimization difficulty. \cite{zhang2017hitting} discuss the difficulty of optimizing RFF directly on different tasks. 
Here we compare IKL-NN and IKL-RFF on challenging Google Billion Word dataset. We train IKL-RFF
with the same setting as Section~\ref{sec:gan_empirical} and Appendix~\ref{sec:gan_hyp}, where we set the pool size $m$
to be $1024$ and the updating schedule between critic and generator to be $10:1$, but we tune the Adam optimization
parameter for IKL-RFF for fair comparison. As discussed above, please note that the number of parameters for $h_\psi$ is
$2048$ while IKL-RFF uses $16384$ when $m=1024$. The results are shown in Table~\ref{tb:ikl_js4}. Even IKL-RFF is using more
parameters, the performance $0.383$ is not competitive as IKL-NN, which achieves $0.372$.

In Algorithm~\ref{alg:mmdgan_IKL}, we update $f_\varphi$ and $h$ in each iteration with $n_c$ times, where we use $n_c=10$
here. We keep the number of updating $f_\varphi$ to be $10$, but increase the number of update for
$\{\hat{\omega_i}\}_{i=1}^1024$ to be $12, 14, 18$ in each iteration. The result is shown in Table~\ref{tb:ikl_js4} with
symbols +2, +4 and +8 respectively. Clearly, we see IKL-RFF need more number of updates to achieve competitive
performance with IKL-NN. The results might implies IKL-RFF is a more difficult optimization problem with more parameters
than IKL-NN. It also confirms the effectiveness of learning implicit generative models with deep neural
networks~\cite{Goodfellow14GAN}, but the underlying theory is still an open research question. 
A better optimization algorithm~\cite{zhang2017hitting} may improve the performance gap between IKL-NN and IKL-RFF, which worth
more study as future work.

\section{Proof of Theorem~\ref{thm:consist}}
\label{sec:pf_consist}
We first prove two Lemmas. 

\begin{lemma} (Consistency with respect to data)
\label{lem:cons_data}
With probability at least $1-\delta$, we have
\[
	\sup_{h\in\Hcal}|\hat{T}(k_h)-T(k_h)| \leq 2\EE_X\bigg[\mathfrak{R}_X^{n-1}(\Fcal_\Hcal)\bigg] + \sqrt{\frac{2}{n}\log{\frac{1}{\delta}}}
\]
\end{lemma}
\begin{proof}
Define 
\[
\rho(x_1, \dots, x_n) = \sup_{h\in\Hcal}|\hat{T}(k_h)-T(k_h)|,
\] 
since $|k_h(x, x')|\leq 1$, it is clearly 
\[
\begin{array}{cl}
	\displaystyle \sup_{x_1,\dots,x_i,x_i',\dots, x_n} & |\rho(x_1, \dots,x_i,\dots x_n) - \\
	& \rho(x_1, \dots,x_i',\dots x_n)| \leq \frac{2}{n}.
\end{array}
\]

Applying McDiarmids Inequality, we get 
\[
	\PP\left( \rho(x_1, \dots, x_n)- \EE[\rho(x_1, \dots, x_n)] \geq \epsilon \right) \leq \exp\left( \frac{-n\epsilon^2}{2}
	\right).
\]

By Lemma~\ref{lem:sup}, we can bound 
\[
\EE[\rho(x_1, \dots, x_n)] \leq 2\EE_X\bigg[\mathfrak{R}_X^{n-1}(\Fcal_\Hcal)\bigg]
\]
and finish the proof.
\end{proof}

\begin{lemma}
\label{lem:sup}
Given $X=\{x_1,\dots,x_n\}$, define  
\[
\rho(x_1, \dots, x_n) = \sup_{h\in\Hcal}|\hat{T}(k_h)-T(k_h)|,
\] 
we have 
\[
	\EE\bigg[ \rho(x_1, \dots, x_n) \bigg] \leq 2\EE_X\bigg[ \mathfrak{R}_X^{n-1}(\Fcal_\Hcal) \bigg],
\]
\end{lemma}
\begin{proof}
The proof is closely followed by~\cite{DziugaiteRG15}.
Given $h$, we first define $t_h(x,x') = s(x, x')k_h(x, x')$ as a new kernel function to simplify the notations. We are
then able to write 
\[
	\begin{array}{cl}
		& \EE\bigg[ \rho(x_1, \dots, x_n) \bigg] \\
	= & \displaystyle \EE_X\bigg[ \sup_{h_\in \Hcal} \left| \EE\bigg[ t_h(z,z') \bigg] - \frac{1}{n(n-1)}\sum_{i\neq j} t_h(x_i, x_j) \right| \bigg] \\
	\leq & \displaystyle \EE_{X, Z}\bigg[ \sup_{h_\in \Hcal} \left| \frac{1}{n(n-1)}\sum_{i\neq j} \left(t_h(z_i, z_j) - t_h(x_i, x_j)\right) \right| \bigg]
	\end{array}
\]
by using Jensen's inequality. Utilizing the conditional expectation and introducing the Rademacher random variables
$\{\sigma_i\}_{i=1}^{n-1}$, we can write the above bound to be 
\begin{align}
	 \displaystyle \frac{1}{n}\sum_i \EE_{X_{-i}, Z_{-i}} \EE_{x_i, z_i} \bigg[ \sup_{h_\in \Hcal} \left| \frac{\sum_{i\neq
	j} t_h(z_i, z_j) - t_h(x_i, x_j)}{n-1} \right| \bigg] \nonumber  \\
	= \displaystyle \EE_{X,Z}\EE_{X',Z',\sigma}\bigg[ \sup_{h\in\Hcal}\left| \frac{1}{n-1}\sum_{i=1}^{n-1}\sigma_i(t_h(z',z_n)-t_h(x',x_n) ) \right| \bigg]
\end{align}
The equality follows by $X-X'$ and $-(X-X')$ has the same distributions if $X$ and $X'$ are independent samples
from the same distribution.
Last, we can bound it by 
\[
	\begin{array}{cl}
		\leq & \displaystyle \EE_{X}\EE_{\sigma, X'} \bigg[ \sup_{h\in \Hcal} \left| \frac{2}{n-1}
		\sum_{i=1}^{n-1}\sigma_it_h(x',x_i) \right| \bigg] \\ 
		\leq & \displaystyle \EE_{X}\EE_\sigma \bigg[ \sup_{f\in \Fcal_{k_\Hcal}}\left|  \frac{2}{n-1}
		\sum_{i=1}^{n-1}\sigma_if(x_i) \right| \bigg] \\
		= & 2\EE_X[ \mathfrak{R}_X^{n-1}(\Fcal_\Hcal)]
	\end{array}
\]
The second inequality follows by $s(x,x')\phi(x') \in \Fcal$ since $|s(x,x')| \leq 1$.

\end{proof}

\begin{lemma} (Consistency with respect to sampling random features)
\label{lem:cons_rff}
With probability $1-\delta$, we have
\[
	|\sup_{h\in\Hcal}\hat{T}(k_h) - \sup_{h\in\Hcal}\hat{T}(\hat{k}_h) | \leq \sqrt{\frac{2\log\frac{4}{\delta}}{m}} 
\]
\end{lemma}

\begin{proof}
Let the optimal solutions be
\[
\begin{array}{ccl}
	h^* & = & \arg\max_{h\in\Hcal}\hat{T}(k_h) \\
	\hat{h} & = & \arg\max_{h\in\Hcal}\hat{T}(\hat{k}_h),
\end{array}
\]
By definition, 
\[
\begin{array}{cl}
	& \hat{T}(\hat{k}_h)  \\
	 = & \displaystyle \frac{1}{n(n-1)}\sum_{i\neq j} s_{ij}\left(\frac{1}{m}\sum_{k=1}^m
	\cos(h(\nu_k)^\top(x_i-x_j))\right) \\
	= & \displaystyle \frac{1}{m}\sum_{k=1}^m\left( \frac{1}{n(n-1)} s_{ij} \cos(h(\nu_k)^\top(x_i-x_j)) \right)
\end{array}
\]
It is true that $|\frac{1}{n(n-1)} s_{ij} \cos(h(\nu_k)^\top(x_i-x_j)|\leq 1$ since $|s_{ij}|<1$ and $|\cos(x)|<1$. 
we then have
\[
\begin{array}{cl}
	& \PP(|\sup_{h\in\Hcal}\hat{T}(k_h) - \sup_{h\in\Hcal}\hat{T}(\hat{k}_h) |>\epsilon) \\
	\leq & \PP(|\hat{T}(k_{h^*}) - \hat{T}(\hat{k}_{h^*}) |>\epsilon) +  \PP(|\hat{T}(k_{\hat{h}}) - \hat{T}(\hat{k}_{\hat{h}})|>\epsilon)\\  
	\leq & \displaystyle 4\exp\left( -\frac{m\epsilon^2}{2} \right),
\end{array}
\]
where the last inequality follows from the Hoeffding's inequality.
\end{proof}

With Lemma~\ref{lem:cons_data} and Lemma~\ref{lem:cons_rff}, we are ready to prove Theorem~\ref{thm:consist}.
We can decompose 
\[
\begin{array}{cl}
	& \displaystyle |T(\hat{k}_{\hat{h}}) - \sup_{h\in \Hcal} T(k_h) | \\ 
	\leq & \displaystyle |\sup_{h\in \Hcal} T(k_h) - \sup_{h\in \Hcal} \hat{T}(k_h)| +  
	|\sup_{h\in \Hcal} \hat{T}(k_h) - \hat{T}(\hat{k}_{\hat{h}})| + |\hat{T}(\hat{k}_{\hat{h}}) - T(\hat{k}_{\hat{h}})| \\
	\leq & \displaystyle \sup_{h\in \Hcal} |T(k_h) - \hat{T}(k_h)| +  
	|\sup_{h\in \Hcal} \hat{T}(k_h) - \hat{T}(\hat{k}_{\hat{h}})|  + \sup_{h\in \Hcal} |\hat{T}(\hat{k}_{h}) - T(\hat{k}_{h})|
\end{array}
\]
We then bound the first and third terms by Lemma~\ref{lem:cons_data} and the second term by Lemma~\ref{lem:cons_rff}.
Last, using a union bound completes the proof. 

\section{Generalization of Random Kitchen Sinks with IKL}
\label{sec:general}

\begin{theorem} (Generalization~\citep{cortes2010generalization})
Define the true and empirical misclassification for a classifier $f$ as
$R(f)  =   \PP(Yf(X)<0)$ and $\hat{R}_\gamma(h)  =  \frac{1}{n}\sum_{i=1}^n\min\left\{ 1,
[1-yf(x_i)/\gamma]_+ \right\}$.
Then 
\[
\displaystyle\sup_{f\in \hat{\Fcal}_\Hcal}\{R(f) - \hat{R}_\gamma(f) \} \leq \frac{2}{\gamma}\mathfrak{R}_X^n( \hat{\Fcal}_\Hcal ) +
3\sqrt{\frac{\log\frac{2}{\delta}}{2n}}
\]
with probability at least $1-\delta$.
\end{theorem}

\section{Hyperparameters}
We report the hyperparameters used in the experiments.
\subsection{GAN}
\label{sec:gan_hyp}
For Gaussian kernels, we use $\sigma_q=\{1,2,4,8,16\}$ for images and $\sigma_q=\{0.5, 1, 2, 4, 8\}$ for text;
for RQ kernels, we use $\alpha_q=\{0.2,0.5,1,2,5\}$ for images and $\alpha_q = \{0.04, 0.1, 0.2, 0.4, 1\}$ for text.
We used Adam as optimizer. 
The learning rate for training both $f_\phi$ and $g_\theta$ is $0.0005$ and $0.0001$ for image and text 
experiments, respectively. The batch size $B$ is $64$. 
We set hyperparameter $n_c$ for updating critic to be $n_c=5$ and $n_c=10$ for CIFAR10 and Google Billion Word datasets.
The learning rate of of $h_\psi$ for Adam is $10^{-6}$.

\subsection{Random Kitchen Sinks with IKL}
\label{sec:classification_hyp}
For OPT-KL, we use the code provided by \cite{sinha2016learning}\footnote{\scriptsize \url{https://github.com/amansinha/learning-kernels}}.
We tune the hyperparameter $\rho = \{1.25, 1.5, 2, 4, 16, 64\}$ on the validation set.
For RFF, OPT-KL, and IKL, the linear classifier is Logistic Regression
\cite{fan2008liblinear}\footnote{\scriptsize \url{https://github.com/cjlin1/liblinear}},
as to make reasonable comparison with MLP.
We use $3$-fold cross validation to select the best $C$
on training set and present the error rate on test set. 
For CIFAR-10 and MNIST, 
we normalize data to be
zero mean and one standard deviation in each feature dimension. 
The learning rate for Adam is $10^{-6}$.
We follow \citet{bullins2017not} to use early stopping when performance on validation set does not
gain.

\end{document}